\documentclass[a4paper,10pt]{article}
\usepackage{stmaryrd}
\usepackage{amsfonts}
\usepackage{bbm}
\usepackage{amscd}
\usepackage{mathrsfs}
\usepackage{latexsym,amssymb,amsmath,amscd,amscd,amsthm,amsxtra}
\usepackage[dvips]{graphicx}
\usepackage[utf8]{inputenc}
\usepackage[T1]{fontenc}
\usepackage{lmodern}
\usepackage{amssymb}
\usepackage[all]{xy}
\usepackage{nicefrac,mathtools,enumitem}
\usepackage{microtype}
\usepackage{CJK}
\usepackage{amssymb}
\usepackage{epstopdf}
\usepackage{graphicx}
\usepackage{graphics}
\usepackage[T1]{fontenc}
\usepackage[utf8]{inputenc}
\usepackage{authblk}
\usepackage{subfigure}
\usepackage{caption}
\usepackage{framed}
\usepackage{mathrsfs}
\usepackage[noend]{algpseudocode}
\usepackage{algorithmicx,algorithm}
\usepackage{framed}
\usepackage{booktabs}
\usepackage{threeparttable}

\numberwithin{equation}{section}
\newtheorem{thm}{Theorem}
\newtheorem{lem}{Lemma}
\newtheorem{cor}{Corollary}
\newtheorem{pro}{Proposition}

\newcommand {\emptycomment}[1]{}

\newcommand{\be }{\begin{equation}}
\newcommand{\ee }{\end{equation}}

\newcommand{\huaE}{\mathcal{E}}
\newcommand{\huaF}{\mathcal{F}}

\newcommand{\huaH}{\mathcal{H}}

\newcommand{\huaT}{\mathcal{T}}

\newcommand{\huaN}{\mathcal{N}}

\newcommand{\nono}{\nonumber}

\allowdisplaybreaks

\newcommand{\f}{\frac}

\newcommand{\aaa}{\alpha}

\newcommand{\nn}{\langle}
\newcommand{\mm}{\rangle}
\newcommand{\ww}{\widetilde}
\newcommand{\la}{\lambda}
\newcommand{\mbb}{\mathbb}

\def\bea{\begin{eqnarray}}
\def\eea{\end{eqnarray}}
\def\be{\begin{equation}}
\def\ee{\end{equation}}
\def\blm{\begin{lem}}
\def\elm{\end{lem}}

\def\b{\mathcal{B}}
\def\a{\mathcal{A}}

\def\c{\mathcal{C}}
\def\o{\mathcal{O}}

\def\e{\mathcal{E}}

\def\tra{\huaT_{1,|D|,\la}}
\def\trb{\huaT_{2,|D|,\la}}
\def\trc{\huaT_{3,|D|,\la}}
\def\trd{\huaT_{4,|D|,\la}}
\def\tre{\huaT_{5,|D|,\la}}
\def\trf{\huaT_{6,|D|,\la}}

\def\g{\gamma}

\def\kk{\kappa}
\def\jj{\big\|}

\def\n{\nono&&}

\def\hha{\hat{\a}_{|D|,\la}}
\def\ha{\a_{|D|,\la}}
\def\hb{\b_{|D|,\la}}
\def\hc{\c_{|D|,\la}}
\def\hd{\mathcal{D}_{|D|,\la}}

\def\hhe{E_{\hat{D},\la,\sigma}}
\def\he{E_{D,\la,\sigma}}

\def\les{\lesssim}

\begin{document}
\title{
{Robust Kernel-based Distribution Regression}
 }\vspace{2mm}
%\author[]{ }
%\author[]{}
%\author[]{}
%%\author[]{\\}
\date{}
%\affil[]{Department of Mathematics, City University of Hong Kong}
%\affil[]{Kowloon, Hong Kong}

%\affil[]{}
%\affil[]{}

\author{Zhan Yu, Daniel W. C. Ho\\ \small Department of Mathematics, City University of Hong Kong\\ \small Kowloon, Hong Kong, Email: zhanyu2-c@my.cityu.edu.hk; madaniel@cityu.edu.hk  % "\small" is optional
        \and
        Ding-Xuan Zhou\\ \small School of Data Science and Department of Mathematics,
City University of Hong Kong\\ \small Kowloon, Hong Kong, Email: mazhou@cityu.edu.hk}

\maketitle

\begin{abstract}
Regularization schemes for regression have been widely studied in learning theory and inverse problems. In this paper, we study distribution regression (DR) which involves two stages of sampling, and aims at regressing from probability measures to real-valued responses over a reproducing kernel Hilbert space (RKHS). Recently, theoretical analysis on DR has been carried out via kernel ridge regression and several learning behaviors have been observed. However, the topic has not been explored and understood beyond the least square based DR. By introducing a robust loss function $l_{\sigma}$ for two-stage sampling problems, we present a novel robust distribution regression (RDR) scheme. With a windowing function $V$ and a scaling parameter $\sigma$ which can be appropriately chosen, $l_{\sigma}$ can include a wide range of popular used loss functions  that enrich the theme of DR. Moreover, the loss $l_{\sigma}$ is not necessarily convex, hence largely improving the former regression class (least square) in the literature of DR. The learning rates under different regularity ranges of the regression function $f_{\rho}$ are comprehensively studied and derived via integral operator techniques. The scaling parameter $\sigma$ is shown to be crucial in providing robustness and satisfactory learning rates of RDR.
\end{abstract}


%$\mathfrak L\mathrm{z}z\mathbf{z}\mathbbm{L}\mathscr{D}{\L}_D$
\section{Introduction}\vspace{0.000000000000000000000001mm}
Data from many applications often appear in the form of functional data or matrix-valued data. Such type of data impose difficulty for classical regression methods to tackle corresponding problems. Hence, developing suitable regression schemes for these problems becomes desirable. Recently, distribution regression (DR) was introduced to handle complicated data defined on some Banach spaces (\cite{fang}, \cite{dr2}, \cite{twostage}, \cite{only}, \cite{mmdr}). In DR, the input data are (probability) distributions on a compact metric space $\bar{X}$.  Generally, the distribution can not be observed directly. On the way of learning the regressor from the distributions to the real valued outputs, we can observe a second-stage sample drawn from the probability measures. For the first stage,  the data set is defined as $\bar{D}=\{(x_i,y_i)\}_{i=1}^{|D|}\subset X\times Y$, in which $|D|$ is cardinality of $\bar{D}$ and each pair $(x_i,y_i)$ is $i.i.d.$ sampled from a meta distribution. $X$ is the input space of probability measures on $\bar{X}$ and $Y=\mathbb R$ is the output space equipped with the standard Euclidean metric. For the second stage, the samples in the sample set $\hat{D}=\big\{(\{x_{i,s}\}_{s=1}^{d_i},y_i)\big\}_{i=1}^{|D|}$ are obtained from distributions $\{x_i\}_{i=1}^{|D|}$ accordingly, where $x_{i,s}\in\bar{X}$. DR has been handled  in many important machine learning and statistical settings, for example, the multi-instance learning, semi-supervised learning and point estimation problems without analytical solution.

The work in this paper is based on a \emph{kernel mean embedding}  ridge regression method for DR (\cite{rkhs}). Let $H=H(k)$ be a reproducing kernel Hilbert space (RKHS) with the associated reproducing kernel $k:\bar{X}\times \bar{X}\rightarrow\mbb R$. Let $(\bar{X},\huaF)$ be a measurable space with $\huaF$ being a Borel $\sigma$-algebra on $\bar{X}$. Denote the set of Borel probability measures on $(\bar{X},\huaF)$ by $\mathcal{M}_1(\huaF)$. Then the \emph{kernel mean embedding} of a distribution $x\in\mathcal{M}_1(\huaF)$ to an element $\mu_x$ of RKHS $H$ is given by
\be
\nono \mu_x=\int_{\bar{X}}k(\cdot,\eta)dx(\eta).
\ee
Via the kernel mean embedding, kernel methods for handling data on Euclidean spaces can be extended to those on the space of probability measures. The kernel mean embedding transformation $x\mapsto\mu_x$ is also injective when $k$ is a characteristic kernel (\cite{fang}, \cite{k2}). The injectivity is shown to be useful in statistical applications (e.g. \cite{k1}, \cite{k3}). Denote the set of the mean embeddings by $X_{\mu}=\{\mu_x:x\in\mathcal{M}_1(\huaF)\}\subseteq H$. Then  the mean embeddings of $\bar{D}$ to $X_{\mu}$ can be represented by
\be
\nono D=\{(\mu_{x_i},y_i)\}_{i=1}^{|D|}.
\ee
Let $\rho$ be the $\mu$-induced probability measure on the product space $Z=X_{\mu}\times Y$. The aim of distribution regression is to predict the conditional mean $Y$ for given $X_{\mu}$ by learning the  regression function defined by %The regression function of $\rho$ w.r.t. the $(\mu_x,y)$-pair is defined by
\be
\nono f_{\rho}(\mu_x)=\int_Yyd\rho(y|\mu_x), \ \mu_x\in X_{\mu},
\ee
in which $\rho(\cdot|\mu_x)$ is the conditional probability measure of $\rho$ induced at $\mu_x\in X_{\mu}$. $f_{\rho}$ is just the minimizer of the least square generalization error
\be
\nono \huaE(f)=\int_Z(f(\mu_x)-y)^2d\rho.
\ee
Generally, in the non-parametric setting, the measure $\rho$ is unknown. In the DR circumstance, the first stage distribution samples $\{x_i\}_{i=1}^{|D|}$ are still unobservable. We are only able to tackle the information of them via the second-stage sample set
\be
\nono \hat{D}=\big\{(\{x_{i,s}\}_{s=1}^{d_i},y_i)\big\}_{i=1}^{|D|}
\ee
with size $d_i\in \mbb N$, $i=1,2,...,|D|$ respectively. Consider a reproducing kernel Hilbert space $(\huaH_K,\|\cdot\|_K)$ associated with a Mercer kernel $K:X_{\mu}\times X_{\mu}\rightarrow\mbb R$. As an extension of one-stage kernel ridge regression scheme (e.g. \cite{pr, s2, pr1, s3, s4, s5, s6}), the  regularized least square DR scheme has the Tikhonov regularized form (\cite{pr}, \cite{fang}, \cite{only})
\be
f_{\hat{D},\la}^{ls}=\arg\min_{f\in\huaH_K}\bigg\{\f{1}{|D|}\sum_{i=1}^{|D|}\big(f(\mu_{\hat{x}_i})-y_i\big)^2+\la\big\|f\big\|_{K}^2\bigg\}, \label{alone}
\ee
in which
\be
\nono \hat{x}_i=\f{1}{d_i}\sum_{s=1}^{d_i}\delta_{x_{i,s}}\\
\ee
serves as the empirical distribution determined by the observable set $\hat{D}=\big\{(\{x_{i,s}\}_{s=1}^{d_i},y_i)\big\}_{i=1}^{|D|}$,
\be
\nono \mu_{\hat{x}_i}=\f{1}{d_i}\sum_{s=1}^{d_i}k(\cdot,x_{i,s})
\ee
is the corresponding kernel mean embedding, $\la>0$ is the regularization parameter. The least square minimization problem \eqref{alone} can be regarded as the Tikhonov regularization solution to an ill-posed inverse problem with noisy data $\hat{D}=\big\{(\{x_{i,s}\}_{s=1}^{d_i},y_i)\big\}_{i=1}^{|D|}$ (\cite{pr}, \cite{inv}).

In this paper, we investigate a more general framework of two-stage distribution regression by considering the novel  regularized RDR scheme:
\be
f_{\hat{D},\la}^{\sigma}=\arg\min_{f\in\huaH_K}\bigg\{\f{\sigma^2}{|D|}\sum_{i=1}^{|D|}V\Big(\f{[f(\mu_{\hat{x}_i})-y_i]^2}{\sigma^2}\Big)+\la\big\|f\big\|_{K}^2\bigg\},\label{algorithm}
\ee
where $V:\mbb R\rightarrow\mbb R$ is a windowing function. In \eqref{algorithm}, the algorithm can also be written in a form of
\be
\nono f_{\hat{D},\la}^{\sigma}=\arg\min_{f\in\huaH_K}\big\{\f{1}{|D|}\sum_{i=1}^{|D|}l_{\sigma}(f(\mu_{\hat{x}_i})-y_i)+\la\|f\|_{K}^2\big\}
\ee
%\be
%\nono f_{\hat{D},\la}=\arg\min_{f\in\huaH_K}\bigg\{\f{1}{|D|}\sum_{i=1}^{|D|}l_{\sigma}\Big(f(\mu_{\hat{x}_i})-y_i\Big)+\la\big\|f\big\|_{K}^2\bigg\},
%\ee
in which the loss function $l_{\sigma}(u)=\sigma^2V(\f{u^2}{\sigma^2})$. It can be witnessed that, in regression strategy \eqref{algorithm}, to enhance robustness of DR, we have replaced the least squares loss by a more general robust alternative generated by windowing function $V$ and scaling parameter $\sigma$. By selecting appropriate windowing function $V$ and scaling parameter $\sigma$, the loss function can cover a wide range of important robust distribution regression (RDR) class, which is new in the literature of DR. For example, Welsch loss $l_{\sigma}(u)=\sigma^2\big[1-\exp(-\f{u^2}{2\sigma^2})\big]$ has been shown to be powerful in settings such as signal processing, data clustering, pattern recognition and non-parameter regression. From a perspective of information-theoretic learning, Welsch loss can also form the basic structure of the well-known correntropy loss, which was first introduced in reference \cite{corf} and is an entropy-based loss. Entropy-based losses mainly include the loss induced by maximum correntropy criterion (MCC) (\cite{fengyl}) and loss induced minimum error entropy  (MEE) criterion (\cite{ht}, \cite{wh}). In addition, there are also many commonly used losses such as the Huber loss (\cite{fw1}), pinball loss (\cite{si}). Recall that the traditional least square DR scheme is a currently most popularly considered DR method in practice. However, it relies only on the mean squared error and belongs to the second-order statistics. Also recall that the least square regression is optimal for Gaussian noise but suboptimal for non-Gaussian noise. In practice, samples may be contaminated by non-Gaussian noise or outliers. Moreover,  least square estimates for regression models are highly sensitive to outliers and when the noise is not Gaussian, it often has a poor performance. Unfortunately, in current two-stage distribution regression area, the commonly considered approach and theoretical study are still limited to least square scheme, no mainstream regression method has been proposed to face non-Gaussian settings yet. These facts motivate us to consider the proposed RDR scheme in \eqref{algorithm} to fill the gap when tackling two-stage distribution regression. Because of the robustness to non-Gaussian noise or outliers, the potential value of the proposed RDR is expected  in two-stage sampling regression.

The goal of  this paper is to investigate  RDR in  a framework of learning theory. To derive learning rates of the estimator $f_{\hat{D},\la}^{\sigma}$ when approximating $f_{\rho}$ and related robustness, we use kernel based integral operator theory as a main tool. Via kernel mean embedding techniques, we learn the regression function $f_{\rho}$ with algorithm \eqref{algorithm} from the given training samples $\hat{D}=\big\{(\{x_{i,j}\}_{j=1}^{d_i},y_i)\big\}_{i=1}^{|D|}$ with $x_{i,1},x_{i,2},...,x_{i,d_i}$ drawn independently from $x_i$. Novel theoretical results on robust estimator $f_{\hat{D},\la}^{\sigma}$ are derived. Note that, in the proposed RDR, the loss function $l_{\sigma}$ involves non-convex functions (e.g. Welsch loss), hence the theoretical study on RDR is essentially different from those on existing DR methods.

We summarize some main contributions of the work. (i) We propose a novel RDR method for two-stage sampling distribution regression.  Learning theory analysis is established on  the estimator $f_{\hat{D},\la}^{\sigma}$ resulted from RDR scheme in \eqref{algorithm}. Novel  error bounds are derived for the estimator $f_{\hat{D},\la}^{\sigma}$ via integral operator techniques. With the introduction of flexibly chosen windowing function $V$ and scaling parameter $\sigma$ that leads to a wide range of commonly used robust loss,  the existing analysis and the class of DR algorithm  in the literature of DR (least square) have  been largely improved. (ii) The learning behavior of RDR is comprehensively explored for regularity index $r$ (introduced below) in any range of $(0,\infty)$. Accordingly, satisfactory convergence rates in terms of sample size $|D|$ are derived. We also show that the optimal mini-max learning rate can be achieved by RDR under appropriate conditions. The significance of $\sigma$ in providing robustness as well as fast learning rates of RDR is shown in our analysis and main results.

\section{Main results}\vspace{0.000000000000000000000001mm}
We assume that, throughout the paper, there exists a constant $M>0$ such that $|y|\leq M$ ($Y\subseteq[-M,M]$) almost surely. $k$ and $K$ are bounded Mercer kernel (symmetric, continuous, positive semidefinite) with bound $B_{k}$ and $\kk$:
\be
 B_{k}=\sup_{v\in\bar{X}}k(v,v)<\infty, \ \kk=\sup_{\mu_u\in X_{\mu}}\sqrt{K(\mu_u,\mu_u)}<\infty. \label{bddker}
\ee
Suppose that $\aaa\in(0,1]$ and $L>0$.  Denote  the Banach space of the bounded linear operators from space $Y=\mbb R$ to $\huaH_K$ by $\mathcal{L}(Y,\huaH_K)$. Let $K_{\mu_x}=K(\mu_x,\cdot)$, $\mu_x\in X_{\mu}$. We treat $K_{\mu_x}$ as an element of $\mathcal{L}(Y,\huaH_K)$ by defining the linear mapping
\be
\nono K_{\mu_x}(y)=yK_{\mu_x}, \  y\in Y.
\ee
The mapping $K_{(\cdot)}:X_{\mu}\rightarrow\mathcal{L}(Y,\huaH_K)$ is assumed to be $(\aaa,L)$-H\"older continuous in following sense
\be
\big\|K_{\mu_x}-K_{\mu_y}\big\|_{\mathcal{L}(Y,\huaH_K)}\leq L\big\|\mu_x-\mu_y\big\|_H^{\aaa}, \ \forall(\mu_x,\mu_y)\in X_{\mu}\times X_{\mu}. \label{holderass}
\ee
According to \cite{only}, the set of mean embeddings $X_{\mu}$ is a separable compact set of continuous functions on $\bar{X}$. Denote the marginal distribution of $\rho$ on $X_{\mu}$ by $\rho_{X_{\mu}}$. Let $L_{\rho_{X_{\mu}}}^2$ be the Hilbert space of square-integrable functions defined on $X_{\mu}$. For $f\in L_{\rho_{X_{\mu}}}^2$, denote the norm of $f$ by
\be
\nono \big\|f\big\|_{L_{\rho_{X_{\mu}}}^2}=\big\nn f,f\big\mm_{\rho_{X_{\mu}}}^{1/2}=\bigg(\int_{X_{\mu}}\big|f(\mu_x)\big|^2d\rho_{X_{\mu}}(\mu_x)\bigg)^{1/2}.
\ee
Define the integral operator $L_K$ on $L_{\rho_{X_{\mu}}}^2$ associated with the Mercer kernel $K: X_{\mu}\times X_{\mu}\rightarrow\mbb R$ by
\be
L_K(f)=\int_{X_{\mu}}K_{\mu_x}f(\mu_x)d\rho_{X_{\mu}}, \ f\in L_{\rho_{X_{\mu}}}^2.
\ee
Since the set of mean embeddings $X_{\mu}$ is compact and $K$ is a Mercer kernel, $L_K$ is a positive compact operator on $L_{\rho_{X_{\mu}}}^2$. Then for any $r>0$, its $r$-th power $L_K^r$ is well defined according to spectral theorem in functional calculus.

Throughout the paper, we assume the following \emph{regularity condition} for the regression function $f_{\rho}$:
\be
f_{\rho}=L_K^r(g_{\rho}) \ \text{for some} \ g_{\rho}\in L_{\rho_{X_{\mu}}}^2, \  r>0. \label{regularity}
\ee
The above assumption means that the regression function lies in the range of operator $L_K^r$, the special case $r=1/2$ corresponds to $f_{\rho}\in\huaH_K$. According to \cite{cucker}, the operator $L_K^{1/2}:\overline{\huaH_K}\rightarrow\huaH_K$ is an isomorphism, in which $\overline{\huaH_K}$ denotes the closure of $\huaH_K$ in $L_{\rho_{X_{\mu}}}^2$. Namely, for any $f\in\overline{\huaH_K}$, $L_K^{1/2}f\in\huaH_K$ and $\|f\|_{L_{\rho_{X_{\mu}}}^2}=\|L_K^{1/2}f\|_K$.

We use the \emph{effective dimension} $\huaN(\la)$ to measure the capacity of $\huaH_K$ with respect to measure $\rho_{X_{\mu}}$, which is defined by the trace of the operator $(\la I+L_K)^{-1}L_K$, that is
\be
\nono \huaN(\la)=\text{Tr}((\la I+L_K)^{-1}L_K), \ \la>0.
\ee
For the effective dimension $\huaN(\la)$, we need the following capacity assumption, which focuses on the $\beta$-rate of  $\huaN(\la)$: there exists a constant $\c_0>0$ independent of $\la$ such that for any $\la>0$,
\be
\huaN(\la)\leq \c_0\la^{-\beta}, \  \text{for some} \ 0<\beta\leq1.  \label{cap}
\ee

Throughout the paper, for the RDR:
\be
\nono f_{\hat{D},\la}^{\sigma}=\arg\min_{f\in\huaH_K}\bigg\{\f{\sigma^2}{|D|}\sum_{i=1}^{|D|}V\Big(\f{[f(\mu_{\hat{x}_i})-y_i]^2}{\sigma^2}\Big)+\la\big\|f\big\|_{K}^2\bigg\},
\ee
we assume that the sample set $D=\{(\mu_{x_i},y_i)\}_{i=1}^{|D|}$ is drawn independently according to probability measure $\rho$. The sample set $\{x_{i,s}\}_{s=1}^{d_i}$ is drawn independently according to probability distribution $x_i$ for $i=1,2,...,|D|$. The windowing function $V$ is assumed to be differentiable with $V_+'(0)=1$ (w.l.o.g. for scaling simplification consideration) and not necessarily convex. It is assumed that there exists some $p$ and $c_p$ such that
\be
|V'(s)-V_+'(0)|\leq c_p s^p, \ \  \forall s>0,   \label{glip}
\ee
and
\be
C_V=\sup_{s\in(0,\infty)}|V'(s)|<\infty \ \text{with} \ V'(s)>0 \ \text{for} \ s>0.
\ee
The proposed regression scheme can cover  many classical loss functions by selecting appropriate $V$ according to above assumption. Among them, some important robust loss functions can also be considered in DR setting to improve the robustness to non-Gaussian noise and outliers. Some well known examples of these loss functions include Welsch loss: $l_{\sigma}(s)=\sigma^2[1-\exp(-\f{s^2}{2\sigma^2})]$, Cauchy loss: $l_{\sigma}(s)=\sigma^2\log(1+\f{s^2}{2\sigma^2})$, Fair loss: $l_{\sigma}(s)=\sigma^2\big[\f{|s|}{\sigma}-\log(1+\f{|s|}{\sigma})\big]$. These examples were first proposed in robust parameter regression setting. It can be witnessed that, the Welsch loss and Cauchy loss are non-convex and satisfy the redescending property.

Now let $f_{\hat{D},\la}^{\sigma}$ be given as in the algorithm \eqref{algorithm},  the main results on the error estimate between $f_{\hat{D},\la}^{\sigma}$ and $f_{\rho}$ are presented in the following theorems. The results are in terms of estimates on the expected difference between $f_{\hat{D},\la}^{\sigma}$ and $f_{\rho}$ in $L_{\rho_{X_{\mu}}}^2$-norm.  The expectation is taken for both $D$ and $\hat{D}$.

Our first main result describes the explicit learning rates of RDR in terms of sample size $|D|$ of data set $D$  and robust scaling parameter $\sigma$. If we assume the capacity condition \eqref{cap}, the following minimax optimal learning rates for RDR \eqref{algorithm} holds.
\begin{thm}\label{thm2}
Suppose that the regularity condition \eqref{regularity} holds with $r>0$ and $|y|\leq M$ almost surely. Assume that the capacity condition \eqref{cap} holds with index $\beta\in(0,1]$ and  the mapping $K_{(\cdot)}:X_{\mu}\rightarrow\mathcal{L}(Y,\huaH_K)$ is $(\aaa,L)$-H\"older continuous with $\aaa\in(0,1]$ and $L>0$. If the sample size in the second stage sampling satisfies $d_1=d_2=\cdots=d_{|D|}=d$, then by choosing
\bea
\la&=&\left\{
  \begin{array}{ll}
    |D|^{-\f{1}{1+\beta}}, & \hbox{$r\in(0,1/2)$;} \\
   |D|^{-\f{1}{2r+\beta}}, & \hbox{$r\in[1/2,1]$;}\\
   |D|^{-\f{1}{2+\beta}}, & \hbox{$r\in(1,\infty)$;}
  \end{array}
\right.
\eea
and
\bea
d&=&\left\{
  \begin{array}{ll}
    |D|^{\f{2}{\aaa(1+\beta)}}, & \hbox{$r\in(0,1/2)$;} \\
   |D|^{\f{1+2r}{\aaa(2r+\beta)}}, & \hbox{$r\in[1/2,1]$;}\\
   |D|^{\f{1}{\aaa}(\f{3}{2+\beta})}, & \hbox{$r\in(1,\infty)$,}
  \end{array}
\right.
\eea
there holds
\bea
\mbb E\bigg[\big\|f_{\hat{D},\la}^{\sigma}-f_{\rho}\big\|_{L_{\rho_{X_{\mu}}}^2}\bigg]&=&\left\{
  \begin{array}{ll}
    \o\Big(\max\Big\{|D|^{-\f{r}{1+\beta}},\f{|D|^{\f{p+1}{1+\beta}}}{\sigma^{2p}}\Big\}\Big), & \hbox{$r\in(0,1/2)$;} \\
   \o\Big(\max\Big\{|D|^{-\f{r}{2r+\beta}},\f{|D|^{\f{p+1}{2r+\beta}}}{\sigma^{2p}}\Big\}\Big), & \hbox{$r\in[1/2,1]$;}\\
   \o\Big(\max\Big\{|D|^{-\f{1}{2+\beta}},\f{|D|^{\f{p+1}{2+\beta}}}{\sigma^{2p}}\Big\}\Big), & \hbox{$r\in(1,\infty)$.}
  \end{array}
\right.
\eea
\end{thm}

In the literature of distribution regression, the existing theoretical studies on learning rates mainly include works  \cite{fang}, \cite{gddr} and \cite{only}. Reference \cite{only} is the first work on theoretical learning rates of least square regressor $f_{\hat{D},\la}^{ls}$ in \eqref{alone}. Reference \cite{only} derives optimal learning rates under the regularity condition $r\in(1/2,1]$ and suboptimal rate when $r=1/2$. The optimal rate of \cite{only} is improved by \cite{fang} to the case $r=1/2$ via a novel integral operator method, that is based on second order decomposition technique for invertible operators in Banach space. Reference \cite{gddr} proposes a kernel based stochastic gradient method in DR setting. Mini-batching is considered for selection of data points in each iteration.  In Theorem \ref{thm2}, we obtain  convergence rates for any value of regularity index $r$ in $(0,\infty)$, in contrast to \cite{fang}, \cite{only} that carries out rate analysis only on $[1/2,1]$. Hence, the convergence rates analysis has been enriched in regularized DR setting. Also, we have introduced the general windowing function $V$ and the scaling parameter $\sigma$. Intuitively, in the explicit bound of Theorem \ref{thm2}, the participation of scaling parameter $\sigma$ indicates its difference with aforementioned works. One difference is that RDR possesses  flexibility on selection of $\sigma$, in contrast to the current DR methods, without the robustness taken into consideration.

There are also many other studies circling around robust learning algorithms in different aspects. For example, references \cite{fw}, \cite{fw1} consider the robust empirical risk minimization scheme. Inspired by convex risk minimization in infinite-dimensional Hilbert spaces, the robustness of support vector regression is extensively investigated in \cite{robust}, \cite{as1}, \cite{as2}, \cite{book}.  The maximum correntropy criterion induced loss is considered in \cite{fengyl} for regression over some compact hypothesis space. Modal regression with robust kernel is studied in \cite{fengf}. The convergence results  in these works are derived under a standard covering number assumption. In contrast to these works, starting from a capacity assumption on  effective dimension $\huaN(\la)$, the paper derived the error bounds and convergence rates via integral operator techniques. On the other hand, based on gradient descent iteration, reference \cite{guos} presents an efficient kernel based robust gradient descent algorithm to learn the regression function $f_{\rho}$. Reference \cite{ht}, \cite{wh} investigate the learning behavior of minimum error entropy algorithm. These works mainly care about the one-stage sampling on data set. For the purpose of treating with distribution samples, we have developed robust regression method for DR setting and provided a selection rule for the regularization parameter $\la$ and the second stage sample size $d$ as in Theorem \ref{thm2}.

The following corollary is a direct consequence of Theorem \ref{thm2}, it shows  that the RDR has nice performance of convergence when the scaling parameter $\sigma$ is chosen to be large enough.

\begin{cor}\label{sigm}
Under the same conditions of Theorem \ref{thm2}, if the scaling parameter $\sigma$ is chosen such that
\bea
\sigma&\geq&\left\{
  \begin{array}{ll}
    |D|^{\f{p+1+r}{2p(1+\beta)}}, & \hbox{$r\in(0,1/2)$;} \\
   |D|^{\f{p+1+r}{2p(2r+\beta)}}, & \hbox{$r\in[1/2,1]$;}\\
   |D|^{\f{p+1+r}{2p(2+\beta)}}, & \hbox{$r\in(1,\infty)$,}
  \end{array}
\right.
\eea
then we have
\bea
\mbb E\bigg[\big\|f_{\hat{D},\la}^{\sigma}-f_{\rho}\big\|_{L_{\rho_{X_{\mu}}}^2}\bigg]&=&\left\{
  \begin{array}{ll}
    \o\Big(|D|^{-\f{r}{1+\beta}}\Big), & \hbox{$r\in(0,1/2)$;} \\
   \o\Big(|D|^{-\f{r}{2r+\beta}}\Big), & \hbox{$r\in[1/2,1]$;}\\
   \o\Big(|D|^{-\f{1}{2+\beta}}\Big), & \hbox{$r\in(1,\infty)$.}
  \end{array}
\right.
\eea
\end{cor}

In a framework of regularized regression, our second main result provides a novel quantitative description on robustness of RDR by considering the expected error between the RDR estimator $f_{\hat{D},\la}^{\sigma}$ (in which the robustness is induced by the scaling parameter $\sigma$) and the classical least square DR estimator $f_{\hat{D},\la}^{ls}$ (without robustness).
\begin{thm}\label{quan}
Let the sample set $D=\{(\mu_{x_i},y_i)\}_{i=1}^{|D|}$ be drawn independently according to probability measure $\rho$. Let $f_{\hat{D},\la}^{ls}$ denote the classical least square DR estimator in \eqref{alone}. Suppose that the sample size in the second stage sampling satisfies $d_1=d_2=\cdots=d$. Then for any given sample size $|D|$, $d$ and regularization parameter $\la>0$, there holds
\be
\mbb E\bigg[\big\|f_{\hat{D},\la}^{\sigma}-f_{\hat{D},\la}^{ls}\big\|_{L_{\rho_{X_{\mu}}}^2}\bigg]\leq\ww{C}\f{(\la^{-(p+\f{1}{2})}+1)(\la^{-\f{3}{2}}d^{-\f{\aaa}{2}}\hha^2+\la^{-\f{1}{2}}\hha)}{\sigma^{2p}}.
\ee
$\ww{C}$ is a constant independent of $D$, $d$, $\la$, $\sigma$ and the explicit form will be given in the proof. $\hha$ is defined by $\hha=\f{\ha}{\sqrt{\la}}+1$ in which $\ha=\f{2\kk}{\sqrt{|D|}}(\f{\kk}{\sqrt{|D|\la}}+\sqrt{\huaN(\la)})$.
\end{thm}
When the sample size $|D|$ and $d$ are large enough, our last main result is a quantitative description on the robust $L_{\rho_{X_{\mu}}}^2$-gap between RDR estimator $f_{\hat{D},\la}^{\sigma}$ and least square DR estimator $f_{\hat{D},\la}^{ls}$.
\begin{cor}\label{gap}
Under same conditions of Theorem \ref{quan}, for any given regularization parameter $\la>0$, there holds
\be
\overline{\lim}_{|D|\rightarrow\infty \atop d\rightarrow\infty}\mbb E\bigg[\big\|f_{\hat{D},\la}^{\sigma}-f_{\hat{D},\la}^{ls}\big\|_{L_{\rho_{X_{\mu}}}^2}\bigg]\leq\ww{C}\f{(\la^{-(p+1)}+\la^{-\f{1}{2}})}{\sigma^{2p}},
\ee
where $\ww{C}$ is a constant to be given explicitly in the proof of Theorem \ref{quan}.
\end{cor}

Recall that $f_{\hat{D},\la}^{\sigma}$ is generated by the introduction of the scaling parameter $\sigma$ that delivers the robustness to the DR scheme. Since the classical least square DR estimator $f_{\hat{D},\la}^{ls}$ does not possess robustness, we know that, when the $L_{\rho_{X_{\mu}}}^2$-distance between $f_{\hat{D},\la}^{\sigma}$ and $f_{\hat{D},\la}^{ls}$ gets smaller, there will be less robustness of the RDR scheme induced by $l_{\sigma}$. In nonparametric regression problems,  to enhance the robustness of RDR, one may choose appropriately small $\sigma$ for use. Actually, in practice, for different purposes, the scaling parameter $\sigma$ may be chosen to be large or small. This idea also matches the work in \cite{fengyl} which handles maximum correntropy criterion. Their work also reveals that too small $\sigma$ would influence the convergence of the regressor  $f_{\hat{D},\la}^{\sigma}$  to $f_{\rho}$. Also, the small $\sigma$ case has been interpreted as modal regression in \cite{fengf}. From above analysis and recent works \cite{fengyl,guos, lv}, we know that, in practice, a moderate scaling parameter $\sigma$ should be chosen appropriately to balance robustness and convergence of RDR.

\section{Sampling operator}
In this section, we provide the analysis and give the notations of the sampling operators for the two stages. We first introduce the following robust regression scheme associated with the first stage sample $D=\{(\mu_{x_i},y_i)\}_{i=1}^{|D|}$.
\be
f_{D,\la}^{\sigma}=\arg\min_{f\in\huaH_K}\bigg\{\f{\sigma^2}{|D|}\sum_{i=1}^{|D|}V\Big(\f{[f(\mu_{x_i})-y_i]^2}{\sigma^2}\Big)+\la\big\|f\big\|_{K}^2\bigg\}.\label{bridge}
\ee
Here, \eqref{bridge}  serves as an important bridge in subsequent proof on two-stage sampling learning.

To start with the theoretical foundation of the paper, we define the sampling operator $S_D:\huaH_K\rightarrow \mbb R^{|D|}$ associated with the first stage sample as
\be
\nono S_Df=(f(\mu_{x_1}),f(\mu_{x_2}),...,f(\mu_{x_{|D|}}))^T, \ f\in\huaH_K,
\ee
and the scaled adjoint operator is given by
\be
\nono S_D^T\mathbf{c}_D=\f{1}{|D|}\sum_{i=1}^{|D|}c_iK_{\mu_{x_i}}, \ \mathbf{c}_D=(c_1,c_2,...,c_{|D|})\in \mathbb R^{|D|}.
\ee
Then  define $L_{K,D}$ as the first stage empirical operator of $L_K$ as follow
\be
\nono L_{K,D}(f)=S_D^TS_D(f)=\f{1}{|D|}\sum_{i=1}^{|D|}f(\mu_{x_i})K_{\mu_{x_i}}=\f{1}{|D|}\sum_{i=1}^{|D|}\nn K_{\mu_{x_i}},f\mm_KK_{\mu_{x_i}}, \ f\in\huaH_K.
\ee
We also define the sampling operator $\hat{S}_D$ associated with the second stage sample as follows.
\be
\nono \hat{S}_Df=(f(\mu_{\hat{x}_1}),f(\mu_{\hat{x}_2}),...,f(\mu_{\hat{x}_{|D|}}))^T, \ f\in\huaH_K,
\ee
Its scaled adjoint operator $\hat{S}_D^T$ is given by
\be
\nono \hat{S}_D^T\mathbf{c}=\f{1}{|D|}\sum_{i=1}^{|D|}c_iK_{\mu_{\hat{x}_i}}, \ \mathbf{c}=(c_1,c_2,...,c_{|D|})\in \mathbb R^{|D|}.
\ee
Then the empirical version operator of $L_{K,D}$ can be defined accordingly by using the second stage sample $\hat{D}$ as follow
\be
L_{K,\hat{D}}(f)=\hat{S}_D^T\hat{S}_D(f)=\f{1}{|D|}\sum_{i=1}^{|D|}f(\mu_{\hat{x}_i})K_{\mu_{\hat{x}_i}}=\f{1}{|D|}\sum_{i=1}^{|D|}\nn K_{\mu_{\hat{x}_i}},f\mm_KK_{\mu_{\hat{x}_i}}, \ f\in\huaH_K. \label{lkd}
\ee
In the following, we use $\mbb E_{\mathbf{z}^{|D|}}[\cdot]$ to denote the expectation w.r.t. $\mathbf{z}^{|D|}=\{z_i=(\mu_{x_i},y_i)\}_{i=1}^{|D|}$. Use $\mbb E_{\mathbf{x}^{\mathbf{d},|D|}|\mathbf{z}^{|D|}}$ to denote the conditional expectation w.r.t. sample $\big\{\{x_{i,s}\}_{s=1}^{d_i}\big\}_{i=1}^{|D|}$ conditioned on $\{z_1,z_2,...,z_{|D|}\}$. Namely
\be
\nono \mbb E_{\mathbf{z}^{|D|}}[\cdot]:=\mbb E_{\{(\mu_{x_i},y_i)\}_{i=1}^{|D|}}[\cdot], \mbb E_{\mathbf{x}^{\mathbf{d},|D|}|\mathbf{z}^{|D|}}[\cdot]:=\mbb E_{\{\{x_{i,s}\}_{s=1}^{d_i}\}_{i=1}^{|D|}\big|\{z_i\}_{i=1}^{|D|}}[\cdot].
\ee
In the following, we denote the output vector by $y=(y_1,y_2,...,y_{|D|})$. The following lemma in \cite{fang} is basic for following proofs on two-stage sampling regression.
\begin{lem}\label{twol}
Suppose the boundedness condition \eqref{bddker} of kernel $k$ and $K$ and $(\aaa,L)$-H\"older condition \eqref{holderass} hold for $K$. Suppose $d_1=d_2=\cdots=d_{|D|}=d$, then
\bea
\n \Big\{\mbb E_{\mathbf{x}^{\mathbf{d},|D|}|\mathbf{z}^{|D|}}\Big[\jj \hat{S}_D^Ty-S_D^Ty\jj_K^2\Big]\Big\}^{\f{1}{2}}\leq(2+\sqrt{\pi})^{\f{1}{2}}ML\f{2^{\f{\aaa}{2}}B_{k}^{\f{\aaa}{2}}}{d^{\f{\aaa}{2}}},\\
\n \Big\{\mbb E_{\mathbf{x}^{\mathbf{d},|D|}|\mathbf{z}^{|D|}}\Big[\jj L_{K,\hat{D}}-L_{K,D}\jj^2\Big]\Big\}^{\f{1}{2}}\leq \kk L(2+\sqrt{\pi})^{\f{1}{2}}\f{2^{\f{\aaa+2}{2}}B_{k}^{\f{\aaa}{2}}}{d^{\f{\aaa}{2}}}.
\eea
\end{lem}

\section{Key analysis and error decomposition}
In this section, we present the key analysis and error decomposition for RDR. These results are crucial to the proof of the main results of the paper. The expectation is taken for both $D$ and $\hat{D}$. In the sequel, we use the following representations:
\be
\ha=\f{2\kk}{\sqrt{|D|}}\bigg(\f{\kk}{\sqrt{|D|\la}}+\sqrt{\huaN(\la)}\bigg);  \label{r1}
\ee
\be
\ha'=\f{1}{|D|\sqrt{\la}}+\f{\sqrt{\huaN(\la)}}{\sqrt{|D|}};  \label{r2}
\ee
\be
\hha=\f{\ha}{\sqrt{\la}}+1. \label{r3}
\ee
\subsection{Basic representations and bounds}
Since we need to handle a more general loss $l_{\sigma}$ instead of previous least square loss, we need to introduce the following quantity. Denote
\be
\nono E_{\hat{D},\la,\sigma}=\f{1}{|D|}\sum_{i=1}^{|D|}\Big[V'\Big(\f{[f_{\hat{D},\la}^{\sigma}(\mu_{\hat{x}_i})-y_i]^2}{\sigma^2}\Big)-V'(0)\Big](f_{\hat{D},\la}^{\sigma}(\mu_{\hat{x}_i})-y_i)K_{\mu_{\hat{x}_i}}
\ee
and
\be
\nono E_{D,\la,\sigma}=\f{1}{|D|}\sum_{i=1}^{|D|}\Big[V'\Big(\f{[f_{D,\la}^{\sigma}(\mu_{x_i})-y_i]^2}{\sigma^2}\Big)-V'(0)\Big](f_{D,\la}^{\sigma}(\mu_{x_i})-y_i)K_{\mu_{x_i}}.
\ee
The following lemma provides a representation for the first-stage and second stage regressor of the RDR. The representations are basic for later use on error decomposition of RDR.
\blm\label{rel}
Let $f_{\hat{D},\la}^{\sigma}$ and $f_{D,\la}^{\sigma}$ be defined as in \eqref{algorithm} and \eqref{bridge}, then they satisfy
\be
 f_{\hat{D},\la}^{\sigma}=(\la I+L_{K,\hat{D}})^{-1}\hat{S}_{D}^Ty-(\la I+L_{K,\hat{D}})^{-1}E_{\hat{D},\la,\sigma} \label{re1}
\ee
and
\be
 f_{D,\la}^{\sigma}=(\la I+L_{K,D})^{-1}S_{D}^Ty-(\la I+L_{K,D})^{-1}E_{D,\la,\sigma}. \label{re2}
\ee
\elm

\begin{proof}
Take Fr\'echet derivative of the regularized functional in \eqref{algorithm}, it follows that
\be
\f{1}{|D|}\sum_{i=1}^{|D|}V'\Big(\f{[f_{\hat{D},\la}^{\sigma}(\mu_{\hat{x}_i})-y_i]^2}{\sigma^2}\Big)(f_{\hat{D},\la}^{\sigma}(\mu_{\hat{x}_i})-y_i)K_{\mu_{\hat{x}_i}}+\la f_{\hat{D},\la}^{\sigma}=0.
\ee
The above relation implies that
\bea
\n \f{1}{|D|}\sum_{i=1}^{|D|}\Big[V'\Big(\f{[f_{\hat{D},\la}^{\sigma}(\mu_{\hat{x}_i})-y_i]^2}{\sigma^2}\Big)-V'(0)\Big](f_{\hat{D},\la}^{\sigma}(\mu_{\hat{x}_i})-y_i)K_{\mu_{\hat{x}_i}}\\
\n+V'(0)\f{1}{|D|}\sum_{i=1}^{|D|}(f_{\hat{D},\la}^{\sigma}(\mu_{\hat{x}_i})-y_i)K_{\mu_{\hat{x}_i}}+\la f_{\hat{D},\la}^{\sigma}=0.
\eea
Substitute the above representation of $\hhe$, use the former definition of $L_{K,\hat{D}}$ and $\hat{S}_D^Ty$, use the condition $V'(0)=1$, it follows that
\be
\nono \hhe+[L_{K,\hat{D}}f_{\hat{D},\la}^{\sigma}-\hat{S}_D^Ty]+\la f_{\hat{D},\la}^{\sigma}=0.
\ee
Namely,
\be
\nono(\la I+L_{K,\hat{D}})f_{\hat{D},\la}^{\sigma}-\hat{S}_D^Ty+\hhe=0.
\ee
Hence we have
\be
\nono f_{\hat{D},\la}^{\sigma}=(\la I+L_{K,\hat{D}})^{-1}\hat{S}_{D}^Ty-(\la I+L_{K,\hat{D}})^{-1}E_{\hat{D},\la,\sigma}.
\ee
The second one follows immediately after replacing $\hat{D}$ by $D$ in above procedures.
\end{proof}
Moreover,  $f_{\hat{D},\la}^{\sigma}$ and $f_{D,\la}^{\sigma}$ has the following upper bound estimate.
\blm\label{fdcon}
the RKHS norm of  $f_{\hat{D},\la}^{\sigma}$ and $f_{D,\la}^{\sigma}$ satisfy
\be
\|f_{\hat{D},\la}^{\sigma}\|_K\leq\sqrt{C_V}M\la^{-1/2},
\ee
and
\be
\|f_{D,\la}^{\sigma}\|_K\leq\sqrt{C_V}M\la^{-1/2}.
\ee
\elm
\begin{proof}
Denote
\be
\e_{\hat{D},\sigma}(f)=\f{\sigma^2}{|D|}\sum_{i=1}^{|D|}V\Big(\f{[f(\mu_{\hat{x}_i})-y_i]^2}{\sigma^2}\Big).
\ee
According to the definition of $f_{\hat{D},\la}^{\sigma}$ in \eqref{algorithm}, we have
\be
\e_{\hat{D},\sigma}(f_{\hat{D},\la}^{\sigma})+\la\|f_{\hat{D},\la}^{\sigma}\|_K^2\leq \e_{\hat{D},\sigma}(0).
\ee
Noting that $V'(s)>0$, then $V(s)>V(0)$ when $s>0$, hence we have $\e_{\hat{D},\sigma}(f_{\hat{D},\la}^{\sigma})>\sigma^2V(0)$, then it follows that
\bea
\n\la\|f_{\hat{D},\la}^{\sigma}\|_K^2\leq\e_{\hat{D},\sigma}(0)-\e_{\hat{D},\sigma}(f_{\hat{D},\la}^{\sigma})\leq\f{\sigma^2}{|D|}\sum_{i=1}^{|D|}V(\f{|y|_i^2}{\sigma^2})-\sigma^2V(0)\\
\n\leq\f{\sigma^2}{|D|}\sum_{i=1}^{|D|}\Big[V(\f{|y_i|^2}{\sigma^2})-V(0)\Big]\leq\f{C_V\sigma^2}{|D|}\sum_{i=1}^{|D|}\f{|y_i|^2}{\sigma^2}\leq C_VM^2.
\eea
Hence, we have $\|f_{\hat{D},\la}^{\sigma}\|_K\leq\sqrt{C_V}M\la^{-1/2}$. Same procedure with $\hat{D}$ replaced by $D$ and $\mu_{\hat{x}_i}$ replaced by $\mu_{x_i}$ implies $\|f_{D,\la}^{\sigma}\|_K\leq\sqrt{C_V}M\la^{-1/2}$.
\end{proof}
The following lemma provides upper bound estimates for  $\he$ and $\hhe$ in RKHS norm.
\blm\label{ees}
$\he$ and $\hhe$ satisfy
\be
\|\hhe\|_K,\ \|\he\|_K\leq2^{2p}c_p\kk\sigma^{-2p}\Big[\kk^{2p+1}(\sqrt{C_V}M)^{2p+1}\la^{-(p+\f{1}{2})}+M^{2p+1}\Big]. \label{edc}
\ee
\elm

\begin{proof}
According to condition \eqref{glip}, we have
\bea
\nono &&\|E_{\hat{D},\la,\sigma}\|_K\leq\f{1}{|D|}\sum_{i=1}^{|D|}\Big|V'\Big(\f{[f_{\hat{D},\la}^{\sigma}(\mu_{\hat{x}_i})-y_i]^2}{\sigma^2}\Big)-V'(0)\Big| |f_{\hat{D},\la}^{\sigma}(\mu_{\hat{x}_i})-y_i|\cdot \|K_{\mu_{\hat{x}_i}}\|_K\\
\n\leq\f{\kk c_p\sigma^{-2p}}{|D|}\sum_{i=1}^{|D|}\Big(\|f_{\hat{D},\la}^{\sigma}\|_{\infty}+|y_i|\Big)^{2p+1}\leq c_p\kk\sigma^{-2p}\Big(\kk\|f_{\hat{D},\la}^{\sigma}\|_K+M\Big)^{2p+1}\\
\n\leq2^{2p}c_p\kk\sigma^{-2p}\Big(\kk^{2p+1}\|f_{\hat{D},\la}^{\sigma}\|_K^{2p+1}+M^{2p+1}\Big)\\
\n\leq2^{2p}c_p\kk\sigma^{-2p}[\kk^{2p+1}(\sqrt{C_V}M)^{2p+1}\la^{-(p+\f{1}{2})}+M^{2p+1}],
\eea
in which the second inequality follows from condition \eqref{glip} and $\|K_{\mu_{\hat{x}_i}}\|_K\leq\kk$. In the third inequality, $\|f_{\hat{D},\la}^{\sigma}\|_{\infty}\leq\kk\|f_{\hat{D},\la}^{\sigma}\|_K$ follows from reproducing property of kernel $K$. The forth inequality follows from the convexity of function $v(x)=x^{2p+1}$, $x>0$. The last inequality follows from Lemma \ref{fdcon}.
\end{proof}

\begin{lem}\label{le2}
Let the sample set $D$ be drawn independently according to probability measure $\rho$. Let $\ha$ be defined as in \eqref{r1}. For a positive continuous function $\Phi$, if with probability at least $1-\delta$, $\delta\in(0,1)$, a random variable $X_{|D|,\la}\geq0$ satisfies $X_{|D|,\la}\leq\Phi(\ha)\log^m\f{2}{\delta}$, $m\in \mbb N_+$. Then we have,
\be
\nono \mbb E_{\mathbf{z}^{|D|}}\bigg[ X_{|D|,\la}^s\bigg]\leq\big(2\Gamma(ms+1)+\log^{ms}2\big)\Phi(\ha)^s,\ s\geq1.
\ee
The same result holds when $\ha$ is replaced by $\ha'$.
\end{lem}
\begin{proof}
The condition implies that, for $0<\delta<2$,
\be
\text{Prob}\Big\{X_{|D|,\la}\leq\Phi(\ha)\log^m\f{4}{\delta}\Big\}\geq1-\f{\delta}{2}.
\ee
Make variable change $\g=\Phi(\ha)^s\log^{ms}\f{4}{\delta}$, $s\geq1$, then it follows that $\g^{\f{1}{s}}=\Phi(\ha)\log^m\f{4}{\delta}$ and $\f{\delta}{2}=2\exp\{-\g^{\f{1}{ms}}/\Phi(\ha)^{1/m}\}$. Note that for $\g>\Phi(\ha)^s\log^{ms}2$, if we set the variable $\xi=X_{|D|,\la}^s$, there holds
\be
\nono \text{Prob}\Big\{\xi>\g\Big\}=\text{Prob}\Big\{\xi^{1/s}>\g^{1/s}\Big\}\leq\f{\delta}{2}=2\exp\Big\{-\f{\g^{1/ms}}{\Phi(\ha)^{1/m}}\Big\}.
\ee
Then by using the formula $\mbb E[\xi]=\int_0^{\infty}\text{Prob}(\xi>\g)d\g$ to $\xi=X_{|D|,\la}^s$, we have
\bea
\n\mbb E_{\mathbf{z}^{|D|}}\Big[\xi\Big]=\int_0^{\Phi(\ha)^s\log^{ms}2}\text{Prob}(\xi>\g)d\g+\int_{\Phi(\ha)^s\log^{ms}2}^{\infty}\text{Prob}(\xi>\g)d\g\\
\n\ \ \ \ \ \quad\quad \ \leq\Phi(\ha)^s\log^{ms}2+\int_{\Phi(\ha)^s\log^{ms}2}^{\infty}2\exp\Big\{-\f{\g^{1/ms}}{\Phi(\ha)^{1/m}}\Big\}d\g.
\eea
Do the variable change $\g=\Phi(\ha)^sx^{ms}$, then the integral in above second term
\be
\nono \int_{\Phi(\ha)^s\log^{ms}2}^{\infty}2\exp\Big\{-\f{\g^{1/ms}}{\Phi(\ha)^{1/m}}\Big\}d\g=2ms\Phi(\ha)^s\int_{\log2}^{\infty}x^{ms-1}e^{-x}dx\leq2\Gamma(ms+1)\Phi(\ha)^s,
\ee
which completes the proof. Using the same procedures with $\ha$ replaced by $\ha'$, we known the inequality holds with  $\ha$ replaced by $\ha'$.
\end{proof}
For handling the error decomposition and corresponding bounds of integral operators later, we denote
\bea
&&\hb=\|(\la I+L_K)^{-\f{1}{2}}(S_D^Ty-L_Kf_{\rho})\|_K,\label{nb}\\
&&\hc=\|(\la I+L_K)(\la I+L_{K,D})^{-1}\|,\label{nc}\\
&&\hd=\|(\la I+L_K)^{-\f{1}{2}}(L_K-L_{K,D})\|. \label{nd}
\eea
From \cite{s2}, \cite{s3}, \cite{guoc}, with probability at least $1-\delta$, the following three estimates hold,
\bea
&&\hb\leq\f{2M(\kk+1)}{\kk}\ha'\log\f{2}{\delta},\label{eb}\\
&&\hc\leq\Big(\f{\ha\log\f{2}{\delta}}{\sqrt{\la}}+1\Big)^2,\label{ec}\\
&&\hd\leq2\ha\log\f{2}{\delta}.\label{ed}
\eea
Then we arrive at the following expected norm bound estimates.
\blm\label{bcd}
Let $\hb$, $\hc$ and $\hd$ be defined as in \eqref{nb}, \eqref{nc} and \eqref{nd}, then they satisfy
\bea
\n\mbb E_{\mathbf{z}^{|D|}}[\hb^s]\leq(2\Gamma(s+1)+\log^{s}2)\Big(\f{2M(\kk+1)}{\kk}\ha'\Big)^{s}, \ s\geq0\\
\n\mbb E_{\mathbf{z}^{|D|}}[\hc^s]\leq(2\Gamma(2s+1)+\log^{2s}2)\hha^{2s}, \ s\geq0,\\
\n\mbb E_{\mathbf{z}^{|D|}}[\hd^s]\leq(2\Gamma(s+1)+\log^{s}2)\Big(2\ha\Big)^s, \ s\geq0,
\eea
in which $\hha=\f{\ha}{\sqrt{\la}}+1$.
\elm
\begin{proof}
Due to the fact that the estimates \eqref{eb}, \eqref{ec} and \eqref{ed} hold with probability of $1-\delta$, then we consider random variables $X_{|D|,\la}=\hb$, $X_{|D|,\la}=\hc$ and $X_{|D|,\la}=\hd$ respectively. If we take function $\Phi_1 (x)=\f{2M(\kk+1)}{\kk}x$, $\Phi_2(x)=\Big(\f{x}{\sqrt{\la}}+1\Big)^2$, and $\Phi_3(x)=2x$  for $\hb$, $\hc$ and $\hd$ respectively,  the desired results are derived after using  Lemma \ref{le2} with $m=1,2,1$.
\end{proof}

\subsection{Basic decomposition and expected bounds}
Denote  $f_{\la}$ as the data-free minimizer for least square regression:
\be
f_{\la}=\arg \min_{f\in\huaH_K}\Big\{\|f-f_{\rho}\|_{L_{\rho_{X_{\mu}}}^2}^2+\la\|f\|_K^2\Big\}.
\ee
It follows from Smale and Zhou \cite{s4} that
\be
f_{\la}=(\la I+L_K)^{-1}L_Kf_{\rho}. \label{ssz}
\ee
According to  the representation  $f_{D,\la}^{\sigma}=(\la I+L_{K,D})^{-1}S_{D}^Ty-(\la I+L_{K,D})^{-1}E_{D,\la,\sigma}$ in Lemma \ref{rel} and $f_{\la}=(\la I+L_K)^{-1}L_Kf_{\rho}$, we have the following decomposition
\bea
\n f_{D,\la}^{\sigma}-f_{\la}\\
\n=(\la I+L_{K,D})^{-1}S_{D}^Ty-(\la I+L_K)^{-1}L_Kf_{\rho}-(\la I+L_{K,D})^{-1}E_{D,\la,\sigma}\\
\n=(\la I+L_{K,D})^{-1}(S_D^Ty-L_Kf_{\rho})+[(\la I+L_{K,D})^{-1}-(\la I+L_{K})^{-1}]L_Kf_{\rho}-(\la I+L_{K,D})^{-1}\he.
\eea
Set $A=\la I+L_{K,D}$ and $B=\la I+L_{K}$, then the fact that $A^{-1}-B^{-1}=A^{-1}(B-A)B^{-1}$ for any invertible operator $A$, $B$ in Banach space and \eqref{ssz} imply that
\bea
\n f_{D,\la}^{\sigma}-f_{\la}\\
\n=(\la I+L_{K,D})^{-1}(S_D^Ty-L_Kf_{\rho})+(\la I+L_{K,D})^{-1}(L_K-L_{K,D})f_{\la}-(\la I+L_{K,D})^{-1}\he.
\eea
Then we have
\bea
\n \max\{\|f_{D,\la}^{\sigma}-f_{\la}\|_{L_{\rho_{X_{\mu}}}^2},\sqrt{\la}\|f_{D,\la}^{\sigma}-f_{\la}\|_K\}\\
\n\leq\|(\la I+L_K)^{1/2}(\la I+L_{K,D})^{-1/2}(\la I+L_{K,D})^{-1/2}(\la I+L_{K})^{1/2}(\la I+L_{K})^{-1/2}(S_D^Ty-L_Kf_{\rho})\|_K \\
\n \ +\|(\la I+L_{K})^{1/2}(\la I+L_{K,D})^{-1/2}(\la I+L_{K,D})^{-1/2}(\la I+L_{K})^{1/2}(\la I+L_{K})^{-1/2}(L_K-L_{K,D})f_{\la}\|_K\\
\n \ +\|(\la I+L_K)^{1/2}(\la I+L_{K,D})^{-1}\he\|_K.
\eea
The basic norm inequality implies that the above norm
\bea
\n \leq \|(\la I+L_K)^{1/2}(\la I+L_{K,D})^{-1/2}\|\|(\la I+L_{K,D})^{-1/2}(\la I+L_{K})^{1/2}\|\|(\la I+L_{K})^{-1/2}(S_D^Ty-L_Kf_{\rho})\|_K\\
\n \  +\|(\la I+L_{K})^{1/2}(\la I+L_{K,D})^{-1/2}\|\|(\la I+L_{K,D})^{-1/2}(\la I+L_{K})^{1/2}\|\|(\la I+L_{K})^{-1/2}(L_K-L_{K,D})\|\|f_{\la}\|_K\\
\n \ +\|(\la I+L_K)^{1/2}(\la I+L_{K,D})^{-1/2}\|\|(\la I+L_{K,D})^{-1/2}\he\|_K\\
\n\leq \hc^{1/2}\hc^{1/2}\hb+\hc^{1/2}\hc^{1/2}\hd\|f_{\la}\|_K+\hc^{1/2}\|(\la I+L_{K,D})^{-1/2}\he\|_K,
%&&\leq \hc\hb+\hc\hd\|f_{\la}\|_K+\hc^{1/2}\f{1}{\sqrt{\la}}\|\he\|_K, \label{cbd}
\eea
in which the above inequalities follows from the fact that $\|T_1^sT_2^s\|\leq\|T_1T_2\|^s$, $s\in (0,1]$ for any two positive self-adjoint operators $T_1$, $T_2$. Noting the fact $(\la I+L_{K,D})^{-1/2}\leq\la^{-1/2}$, we have
\bea
\n\max\{\|f_{D,\la}^{\sigma}-f_{\la}\|_{L_{\rho_{X_{\mu}}}^2},\sqrt{\la}\|f_{D,\la}^{\sigma}-f_{\la}\|_K\}\\
&&\leq \hc\hb+\hc\hd\|f_{\la}\|_K+\hc^{1/2}\f{1}{\sqrt{\la}}\|\he\|_K, \label{cbd}
\eea
Now it is ready to present following proposition on expected error bounds for  $\|f_{D,\la}^{\sigma}-f_{\la}\|_{L_{\rho_{X_{\mu}}}^2}$.
\begin{pro}\label{ones}
There holds
\bea
\n\mbb E_{\mathbf{z}^{|D|}}\big[\|f_{D,\la}^{\sigma}-f_{\la}\|_{L_{\rho_{X_{\mu}}}^2}\big]\\
\n\leq C_{p,\kk,C_V,M}\Bigg[\Big(\f{\ha}{\sqrt{\la}}+1\Big)^2\ha'+\Big(\f{\ha}{\sqrt{\la}}+1\Big)^2\ha\|f_{\la}\|_K\\
\n\ \ +\Big(\f{\ha}{\sqrt{\la}}+1\Big)\f{1}{\sqrt{\la}}\sigma^{-2p}(\la^{-(p+\f{1}{2})}+1)\Bigg].
\eea
The constant $C_{p,\kk,C_V,M}$ will be given in the proof.
\end{pro}

\begin{proof}
From \eqref{cbd}, take expectation and using Cauchy inequality in expectation version, it follows that
\bea
\n \mbb E_{\mathbf{z}^{|D|}}\big[\|f_{D,\la}^{\sigma}-f_{\la}\|_{L_{\rho_{X_{\mu}}}^2}\big]\\
\n\leq\Big\{\mbb E_{\mathbf{z}^{|D|}}[\hc^2]\Big\}^{1/2}\Big\{\mbb E_{\mathbf{z}^{|D|}}[\hb^2]\Big\}^{1/2}+\Big\{\mbb E_{\mathbf{z}^{|D|}}[\hc^2]\Big\}^{1/2}\Big\{\mbb E_{\mathbf{z}^{|D|}}[\hd^2]\Big\}^{1/2}\|f_{\la}\|_K\\
\n+[2\Gamma(2)+\log2]\Big(\f{\ha}{\sqrt{\la}}+1\Big)\f{1}{\sqrt{\la}}2^{2p}c_p\kk\sigma^{-2p}[\kk^{2p+1}(\sqrt{C_V}M)^{2p+1}\la^{-(p+\f{1}{2})}+M^{2p+1}].
\eea
Apply  Lemma \ref{bcd} to above inequality, then the desired result holds after we take out and combine corresponding coefficients in following form
\bea
\n C_{p,\kk,C_V,M}\\
\n=\bigg[(2\Gamma(5)+\log^42)^{1/2}(2\Gamma(3)+\log^22)^{1/2}+(2\Gamma(9)+\log^82)^{1/4}(2\Gamma(5)+\log^42)^{1/4}\bigg]\f{2M(\kk+1)}{\kk}\\
\n \ \ +\bigg[(2\Gamma(5)+\log^42)^{1/2}(2\Gamma(3)+\log^22)^{1/2}+(2\Gamma(9)+\log^82)^{1/4}(2\Gamma(5)+\log^42)^{1/4}\bigg]\cdot2\\
&& \ \ +\bigg[(2\Gamma(2)+\log2)+(2\Gamma(3)+\log^22)^{1/2}\bigg](2^{2p}c_p\kk^{2p+2}(\sqrt{C_V}M)^{2p+2}+2^{2p}c_p\kk M^{2p+1}). \label{dac}
\eea
\end{proof}
The following proposition provides a basic estimate on $\Big\{\mbb E_{\mathbf{z}^{|D|}}\big[\|f_{D,\la}^{\sigma}\|_{K}^2\big]\Big\}^{1/2}$, which will be used to prove the main results.
\begin{pro}\label{fdkb}
There holds
\bea
\n \Big\{\mbb E_{\mathbf{z}^{|D|}}\big[\|f_{D,\la}^{\sigma}\|_{K}^2\big]\Big\}^{1/2}\leq 2C_{p,\kk,C_V,M}\Bigg[\Big(\f{\ha}{\sqrt{\la}}+1\Big)^2\f{\ha'}{\sqrt{\la}}+\Big(\f{\ha}{\sqrt{\la}}+1\Big)^2\f{\ha}{\sqrt{\la}}\|f_{\la}\|_K\\
\n \quad\quad\quad\quad\quad\quad\quad\quad\quad\quad\quad+\Big(\f{\ha}{\sqrt{\la}}+1\Big)\sigma^{-2p}(\la^{-(p+\f{3}{2})}+\la^{-1})\Bigg]+2\|f_{\la}\|_K,
\eea
with  $C_{p,\kk,C_V,M}$ as  in \eqref{dac}.
\end{pro}

\begin{proof}
From \eqref{cbd}, we have
\be
\|f_{D,\la}^{\sigma}-f_{\la}\|_K\leq\f{1}{\sqrt{\la}}\hc\hb+\f{1}{\sqrt{\la}}\hc\hd\|f_{\la}\|_K+\f{1}{\sqrt{\la}}\hc^{1/2}\f{1}{\sqrt{\la}}\|\he\|_K.
\ee
For saving space, we denote the right hand side of \eqref{edc} by $e(\la)$. Since
\be
\|f_{D,\la}^{\sigma}\|_K\leq\|f_{D,\la}^{\sigma}-f_{\la}\|_K+\|f_{\la}\|_K,
\ee
it follows that
\bea
\n\|f_{D,\la}^{\sigma}\|_K\leq \f{1}{\sqrt{\la}}\hc\hb+\f{1}{\sqrt{\la}}\hc\hd\|f_{\la}\|_K+\f{1}{\sqrt{\la}}\hc^{1/2}\f{1}{\sqrt{\la}}e(\la)+\|f_{\la}\|_K.
\eea
Use basic inequality $(\sum_{i=1}^4a_i)^2\leq4(\sum_{i=1}^4a_i^2)$, we have
\be
\|f_{D,\la}^{\sigma}\|_K^2\leq\f{4}{\la}\Big\{\hc^2\hb^2+\hc^2\hd^2\|f_{\la}\|_K^2+\hc\f{1}{\la}e(\la)^2\Big\}+4\|f_{\la}\|_K^2.
\ee
Take expectation, it follows that
\bea
\n\mbb E_{\mathbf{z}^{|D|}}[\|f_{D,\la}^{\sigma}\|_K^2]\leq\f{4}{\la}\Bigg(\Big\{\mbb E_{\mathbf{z}^{|D|}}\big[\hc^4\big]\Big\}^{1/2}\Big\{\mbb E_{\mathbf{z}^{|D|}}\big[\hb^4\big]\Big\}^{1/2}+\Big\{\mbb E_{\mathbf{z}^{|D|}}\big[\hc^4\big]\Big\}^{1/2}\Big\{\mbb E_{\mathbf{z}^{|D|}}\big[\hd^4\big]\Big\}^{1/2}\|f_{\la}\|_K^2\\
\n\quad\quad\quad\quad\quad\quad\quad\quad\quad\quad+\Big\{\mbb E_{\mathbf{z}^{|D|}}\big[\hc\big]\Big\}\f{1}{\la}e(\la)^2\Bigg)+4\|f_{\la}\|_K^2.
\eea
From the basic fact $\sqrt{\sum_{i=1}^4a_i}\leq\sum_{i=1}^4\sqrt{a_i}$ for any positive number $a_i$, $i=1,2,3,4$, we have
\bea
\n\Big\{\mbb E_{\mathbf{z}^{|D|}}\big[\|f_{D,\la}^{\sigma}\|_{K}^2\big]\Big\}^{1/2}\leq\f{2}{\sqrt{\la}}\Big\{\mbb E_{\mathbf{z}^{|D|}}\big[\hc^4\big]\Big\}^{1/4}\Big\{\mbb E_{\mathbf{z}^{|D|}}\big[\hb^4\big]\Big\}^{1/4}\\
\n \quad\quad\quad\quad\quad\quad\quad\quad\quad\quad\quad+\f{2}{\sqrt{\la}}\Big\{\mbb E_{\mathbf{z}^{|D|}}\big[\hc^4\big]\Big\}^{1/4}\Big\{\mbb E_{\mathbf{z}^{|D|}}\big[\hd^4\big]\Big\}^{1/4}\|f_{\la}\|_K\\
\n \quad\quad\quad\quad\quad\quad\quad\quad\quad\quad\quad+\f{2}{\sqrt{\la}}\Big\{\mbb E_{\mathbf{z}^{|D|}}\big[\hc\big]\Big\}^{1/2}\f{1}{\sqrt{\la}}e(\la)+2\|f_{\la}\|_K.
\eea
According to Lemma \ref{bcd}, we have
\bea
\n\Big\{\mbb E_{\mathbf{z}^{|D|}}\big[\|f_{D,\la}^{\sigma}\|_{K}^2\big]\Big\}^{1/2}\\
\n\leq\f{2}{\sqrt{\la}}(2\Gamma(9)+\log^82)^{1/4}(2\Gamma(5)+\log^42)^{1/4}\Big(\f{\ha}{\sqrt{\la}}+1\Big)^2\f{2M(\kk+1)}{\kk}\ha'\\
\n+\f{2}{\sqrt{\la}}(2\Gamma(9)+\log^82)^{1/4}(2\Gamma(5)+\log^42)^{1/4}\Big(\f{\ha}{\sqrt{\la}}+1\Big)^2\cdot2\ha\|f_{\la}\|_K\\
\n+\f{2}{\sqrt{\la}}(2\Gamma(3)+\log^22)^{1/2}\Big(\f{\ha}{\sqrt{\la}}+1\Big)\f{1}{\sqrt{\la}}\Bigg(2^{2p}c_p\kk\sigma^{-2p}\Big[\kk^{2p+1}(\sqrt{C_V}M)^{2p+1}\la^{-(p+\f{1}{2})}+M^{2p+1}\Big]\Bigg)\\
\n+2\|f_{\la}\|_K.
\eea
By taking $C_{p,\kk,C_V,M}$ as in \eqref{dac}, the desired result follows.
\end{proof}
\subsection{Estimates in second-stage sampling and the proof of general error bounds}
The following estimate on $\|L_K^{1/2}(\la I+ L_{K,\hat{D}})^{-1}\|$ is basic for proposed two-stage RDR.
\blm\label{secl}
There holds
\be
\Big\{\mbb E_{\mathbf{x}^{\mathbf{d},|D|}|\mathbf{z}^{|D|}}[\|L_K^{1/2}(\la I+ L_{K,\hat{D}})^{-1}\|^2]\Big\}^{1/2}\leq\Big(\sqrt{2}\la^{-\f{3}{2}}\kk(2+\sqrt{\pi})^{\f{1}{2}}L\f{2^{\f{\aaa+2}{2}}B_k^\f{\aaa}{2}}{d^{\f{\aaa}{2}}}\Big)\hc+\sqrt{2}\la^{-1/2}\hc^{1/2}.
\ee
\elm
\begin{proof}
Start with the following decomposition
\bea
\n \|L_K^{1/2}(\la I+ L_{K,\hat{D}})^{-1}\|\\
\n\leq\|L_K^{1/2}[(\la I+L_{K,\hat{D}})^{-1}-(\la I+L_{K,D})^{-1}]\|+\|L_K^{1/2}(\la I+L_{K,D})^{-1}\|\\
\n\leq\|L_K^{1/2}(\la I+L_{K,D})^{-1}(L_{K,D}-L_{K,\hat{D}})(\la I+L_{K,\hat{D}})^{-1}\|+\|L_K^{1/2}(\la I+ L_{K,D})^{-1}\|\\
\n\leq\|(\la I+L_K)^{1/2}(\la I+L_{K,D})^{-1/2}(\la I+L_{K,D})^{-1/2}(\la I+L_K)^{1/2}(\la I+L_K)^{-1/2}\\
\n\ \ \ (L_{K,D}-L_{K,\hat{D}})(\la I+L_{K,\hat{D}})^{-1}\|+\|(\la I+L_K)^{1/2}(\la I+L_{K,D})^{-1/2}(\la I+L_{K,D})^{-1/2}\|\\
\n\leq \|(\la I+L_K)^{1/2}(\la I+L_{K,D})^{-1/2}\|^2\|(\la I+L_K)^{-1/2}\|\|L_{K,D}-L_{K,\hat{D}}\|\|(\la I+L_{K,\hat{D}})^{-1}\|\\
\n\ \ \ +\|(\la I+L_K)^{1/2}(\la I+L_{K,D})^{-1/2}\|\|(\la I+L_{K,D})^{-1/2}\|\\
\n\leq\hc\la^{-1/2}\la^{-1}\|L_{K,D}-L_{K,\hat{D}}\|+\la^{-1/2}\hc^{1/2},
\eea
in which we have used the fact that $\|T_1^sT_2^s\|\leq\|T_1T_2\|^s$, $s\in (0,1]$ for any two positive self-adjoint operators $T_1$, $T_2$. Then it follows that
\be
\|L_K^{1/2}(\la I+ L_{K,\hat{D}})^{-1}\|^2\leq2\la^{-3}\hc^2\|L_{K,D}-L_{K,\hat{D}}\|^2+2\la^{-1}\hc.
\ee
Take expectation on both sides of above inequality, it follows that
\bea
\n\mbb E_{\mathbf{x}^{\mathbf{d},|D|}|\mathbf{z}^{|D|}}[\|L_K^{1/2}(\la I+ L_{K,\hat{D}})^{-1}\|^2]\leq2\la^{-3}\hc^2E_{\mathbf{x}^{\mathbf{d},|D|}|\mathbf{z}^{|D|}}[\|L_{K,D}-L_{K,\hat{D}}\|^2]+2\la^{-1}\hc\\
\n\leq\Big(2\la^{-3}\kk^2(2+\sqrt{\pi})L^2\f{2^{\aaa+2}B_k^{\aaa}}{d^{\aaa}}\Big)\hc^2+2\la^{-1}\hc,
\eea
in which the second inequality follows from Lemma \ref{twol}.
Due to the basic fact $\sqrt{a_1+a_2}\leq\sqrt{a_1}+\sqrt{a_2}$ for any two positive numbers $a_1, a_2$,  it follows that
\be
\nono\Big\{\mbb E_{\mathbf{x}^{\mathbf{d},|D|}|\mathbf{z}^{|D|}}[\|L_K^{1/2}(\la I+ L_{K,\hat{D}})^{-1}\|^2]\Big\}^{1/2}\leq\Big(\sqrt{2}\la^{-\f{3}{2}}\kk(2+\sqrt{\pi})^{\f{1}{2}}L\f{2^{\f{\aaa+2}{2}}B_k^\f{\aaa}{2}}{d^{\f{\aaa}{2}}}\Big)\hc+\sqrt{2}\la^{-1/2}\hc^{1/2},
\ee
which completes the proof.
\end{proof}
To prove Theorem \ref{thm1}, we make a decomposition on $f_{\hat{D},\la}^{\sigma}-f_{\rho}$,
\be
f_{\hat{D},\la}^{\sigma}-f_{\rho}=(f_{\hat{D},\la}^{\sigma}-f_{D,\la}^{\sigma})+(f_{D,\la}^{\sigma}-f_{\la})+(f_{\la}-f_{\rho}).
\ee
Basic norm triangle inequality implies
\be
\|f_{\hat{D},\la}^{\sigma}-f_{\rho}\|_{L_{\rho_{X_{\mu}}}^2}=\|f_{\hat{D},\la}^{\sigma}-f_{D,\la}^{\sigma}\|_{L_{\rho_{X_{\mu}}}^2}+\|f_{D,\la}^{\sigma}-f_{\la}\|_{L_{\rho_{X_{\mu}}}^2}+\|f_{\la}-f_{\rho}\|_{L_{\rho_{X_{\mu}}}^2}.
\ee
We have already obtained expected bounds for $\|f_{D,\la}^{\sigma}-f_{\la}\|_{L_{\rho_{X_{\mu}}}^2}$ in Proposition \ref{ones}. Also, due to Smale and Zhou \cite{s4},
\be
 \|f_{\la}-f_{\rho}\|_{L_{\rho_{X_{\mu}}}^2}\leq\|g_{\rho}\|_{L_{\rho_{X_{\mu}}}^2}\la^{\min\{r,1\}},\label{sz}
\ee
thing left is to consider estimate for $\|f_{\hat{D},\la}^{\sigma}-f_{D,\la}^{\sigma}\|_{L_{\rho_{X_{\mu}}}^2}$ in this part. We have the following result.
\begin{pro}\label{hfjf}
Suppose that the regularity condition \eqref{regularity} holds with $r>0$ and $|y|\leq M$ almost surely, and the mapping $K_{(\cdot)}:X_{\mu}\rightarrow\mathcal{L}(Y,\huaH_K)$ is $(\aaa,L)$-H\"older continuous with $\aaa\in(0,1]$ and $L>0$. Then we have
\bea
\nono&&\mbb E\bigg[\big\|f_{\hat{D},\la}^{\sigma}-f_{D,\la}^{\sigma}\big\|_{L_{\rho_{X_{\mu}}}^2}\bigg]\\
\nono&&\leq \bar{C}\Bigg\{\big(\f{1}{\la d^{\f{\aaa}{2}}}\hha^2+\hha\big)\f{1}{\la^{\f{1}{2}}d^{\f{\aaa}{2}}}+\f{1}{\la^{\f{1}{2}}d^{\f{\aaa}{2}}}(\f{1}{\la d^{\f{\aaa}{2}}}+1)(\hha^2+\hha)\bigg[\hha^2\cdot\f{\ha'}{\sqrt{\la}}\\
\n \ +\hha^2(\hha-1)\|f_{\la}\|_K+\hha\sigma^{-2p}(\la^{-(p+\f{3}{2})}+\la^{-1})+\|f_{\la}\|_K\bigg]\\
\n \ +\f{1}{\la^{\f{1}{2}}d^{\f{\aaa}{2}}}\f{1}{\la d^{\f{\aaa}{2}}}\sigma^{-2p}(\la^{-(p+\f{3}{2})}+\la^{-1})(\hha^2+\hha)+\hha\sigma^{-2p}(\la^{-(p+1)}+\la^{-\f{1}{2}})\Bigg\}.
\eea
$\bar{C}$ is a constant independent of $D$, $d$, and $\sigma$, it will be specified in the proof.  $\ha=\f{2\kk}{\sqrt{|D|}}\bigg(\f{\kk}{\sqrt{|D|\la}}+\sqrt{\huaN(\la)}\bigg)$, $\ha'=\f{1}{|D|\sqrt{\la}}+\f{\sqrt{\huaN(\la)}}{\sqrt{|D|}}$ and $\hha=\f{\ha}{\sqrt{\la}}+1$.
\end{pro}
\begin{proof}
We decompose $f_{\hat{D},\la}^{\sigma}-f_{D,\la}^{\sigma}$ as
\be
\nono f_{\hat{D},\la}^{\sigma}-f_{D,\la}^{\sigma}=I_1-I_2,
\ee
in which
\bea
&& I_1=(\la I+L_{K,\hat{D}})^{-1}\hat{S}_D^Ty-(\la I+L_{K,D})^{-1}S_D^Ty,\\ \label{i1}
&& I_2=(\la I+L_{K,\hat{D}})^{-1}\hhe-(\la I+L_{K,D})^{-1}\he. \label{i2}
\eea
It is easy to see $\|f_{\hat{D},\la}^{\sigma}-f_{D,\la}^{\sigma}\|_{L_{\rho_{X_{\mu}}}^2}\leq\|I_1\|_{L_{\rho_{X_{\mu}}}^2}+\|I_2\|_{L_{\rho_{X_{\mu}}}^2}$ and  $\mbb E[\|f_{\hat{D},\la}^{\sigma}-f_{D,\la}^{\sigma}\|_{L_{\rho_{X_{\mu}}}^2}]\leq\mbb E[\|I_1\|_{L_{\rho_{X_{\mu}}}^2}]+\mbb E[\|I_2\|_{L_{\rho_{X_{\mu}}}^2}]$ after taking expectation. In the following, we estimate $\mbb E[\|I_1\|_{L_{\rho_{X_{\mu}}}^2}]$ and $\mbb E[\|I_2\|_{L_{\rho_{X_{\mu}}}^2}]$ respectively. We first estimate $\mbb E[\|I_1\|_{L_{\rho_{X_{\mu}}}^2}]$. Since
\bea
\n I_1=(\la I+L_{K,\hat{D}})^{-1}\hat{S}_D^Ty-(\la I+L_{K,D})^{-1}S_D^Ty\\
\n=(\la I+L_{K,\hat{D}})^{-1}(\hat{S}_D^Ty-S_D^Ty)+\Big[(\la I+L_{K,\hat{D}})^{-1}-(\la I+L_{K,D})^{-1}\Big]S_D^Ty\\
\n=(\la I+L_{K,\hat{D}})^{-1}(\hat{S}_D^Ty-S_D^Ty)+(\la I+L_{K,\hat{D}})^{-1}(L_{K,D}-L_{K,\hat{D}})(\la I+L_{K,D})^{-1}S_D^Ty.
\eea
Then, it follows that
\be
\|I_1\|_{L_{\rho_{X_{\mu}}}^2}\leq\|I_{1,1}\|_K+\|I_{1,2}\|_K \  \  \text{and} \ \ \  \mbb E[\|I_1\|_{L_{\rho_{X_{\mu}}}^2}]\leq\mbb E[\|I_{1,1}\|_K]+\mbb E[\|I_{1,2}\|_K],
\ee
in which
\bea
&&I_{1,1}=L_K^{1/2}(\la I+L_{K,\hat{D}})^{-1}(\hat{S}_D^Ty-S_D^Ty),\\
&&I_{1,2}=L_K^{1/2}(\la I+L_{K,\hat{D}})^{-1}(L_{K,D}-L_{K,\hat{D}})(\la I+L_{K,D})^{-1}S_D^Ty.
\eea
Then, using Cauchy inequality, Lemma \ref{twol} and Lemma \ref{secl}, we have
\bea
\n\mbb E[\|I_{1,1}\|_{L_{\rho_{X_{\mu}}}^2}]\leq\mbb E_{\mathbf{z}^{|D|}}\Bigg[\Big\{\mbb E_{\mathbf{x}^{\mathbf{d},|D|}|\mathbf{z}^{|D|}}[\|L_K^{1/2}(\la I+ L_{K,\hat{D}})^{-1}\|^2]\Big\}^{1/2}\Big\{E_{\mathbf{x}^{\mathbf{d},|D|}|\mathbf{z}^{|D|}}\Big[\jj \hat{S}_D^Ty-S_Dy\jj_K^2\Big]\Big\}^{\f{1}{2}}\Bigg] \\
\n\leq\Bigg[\Big(\sqrt{2}\la^{-3/2}(2+\sqrt{\pi})L\f{2^{\f{2+\aaa}{2}}B_k^{\f{\aaa}{2}}}{d^{\f{\aaa}{2}}}\Big)\mbb E_{\mathbf{z}^{|D|}}[\hc]+\sqrt{2}\la^{-1/2}\mbb E_{\mathbf{z}^{|D|}}[\hc^{1/2}]\Bigg]\cdot(2+\sqrt{\pi})^{1/2}LM\f{2^{\f{\aaa}{2}}B_k^{\f{\aaa}{2}}}{d^{\f{\aaa}{2}}}\\
\n\leq\Bigg[\sqrt{2}(2+\sqrt{\pi})^{\f{1}{2}}L2^{\f{\aaa+2}{2}}B_k^{\f{\aaa}{2}}\f{1}{\la^{\f{3}{2}}d^{\f{\aaa}{2}}}(2\Gamma(3)+\log^22)\Big(\f{\ha}{\sqrt{\la}}+1\Big)^2\\
 && \ \ \ \ \ \ \ \ \ \ \ \ \ \ \ \ +\sqrt{2}\la^{-\f{1}{2}}(2\Gamma(2)+\log2)\Big(\f{\ha}{\sqrt{\la}}+1\Big)\Bigg](2+\sqrt{\pi})^{1/2}LM2^{\f{\aaa}{2}}B_k^{\f{\aaa}{2}}\f{1}{d^{\f{\aaa}{2}}}\label{i11},
\eea
in which the first inequality follows from Cauchy inequality, the second inequality follows from Lemma \ref{secl} and Lemma \ref{twol}, the third inequality follows from Lemma \ref{bcd}. On the other hand, following \eqref{re2} in Lemma \ref{rel}, we have
\be
(\la I+L_{K,D})^{-1}S_D^Ty=f_{D,\la}^{\sigma}+(\la I+L_{K,D})^{-1}\he.
\ee
Hence
\be
I_{1,2}=L_K^{1/2}(\la I+L_{K,\hat{D}})^{-1}(L_{K,D}-L_{K,\hat{D}})\Big[f_{D,\la}^{\sigma}+(\la I+L_{K,D})^{-1}\he)\Big].
\ee
Then it follows that
\be
\|I_{1,2}\|_K\leq\|L_K^{1/2}(\la I+L_{K,\hat{D}})^{-1}\|\|L_{K,D}-L_{K,\hat{D}}\|\Big(\|f_{D,\la}^{\sigma}\|_K+\|(\la I+L_{K,D})^{-1}\he\|_K\Big).
\ee
Then by taking expectation, it follows that
\bea
\n\mbb E[\|I_{1,2}\|_K]\leq\mbb E_{\mathbf{z}^{|D|}}\Bigg[\Big\{\mbb E_{\mathbf{x}^{\mathbf{d},|D|}|\mathbf{z}^{|D|}}[\|L_K^{1/2}(\la I+ L_{K,\hat{D}})^{-1}\|^2]\Big\}^{1/2}\Big\{\mbb E_{\mathbf{x}^{\mathbf{d},|D|}|\mathbf{z}^{|D|}}\Big[\jj L_{K,\hat{D}}-L_{K,D}\jj_K^2\Big]\Big\}^{\f{1}{2}}\\
\n\quad \quad \quad \quad \quad \quad \quad \quad \quad \quad \quad \quad \quad \quad \quad \Big(\|f_{D,\la}^{\sigma}\|_K+\|(\la I+L_{K,D})^{-1}\he\|_K\Big)\Bigg]\\
\n\leq\kk L(2+\pi)^{\f{1}{2}}\f{2^{\f{\aaa+2}{2}}B_k^{\f{\aaa}{2}}}{d^{\f{\aaa}{2}}}\Big[\sqrt{2}(2+\sqrt{\pi})^{\f{1}{2}}L2^{\f{\aaa+2}{2}}B_k^{\f{\aaa}{2}}\f{1}{\la^{\f{3}{2}}d^{\f{\aaa}{2}}}+\sqrt{2}\la^{-\f{1}{2}}\Big]\\
\n\ \ \cdot\Big(\{\mbb E_{\mathbf{z}^{|D|}}[\hc^2]\}^{1/2}\{\mbb E_{\mathbf{z}^{|D|}}[\|f_{D,\la}^{\sigma}\|_K^2]\}^{1/2}+\{\mbb E_{\mathbf{z}^{|D|}}[\hc]\}^{1/2}\{\mbb E_{\mathbf{z}^{|D|}}[\|f_{D,\la}^{\sigma}\|_K^2]\}^{1/2}\Big)\\
\n+\kk L(2+\pi)^{\f{1}{2}}\f{2^{\f{\aaa+2}{2}}B_k^{\f{\aaa}{2}}}{d^{\f{\aaa}{2}}}\Big[\sqrt{2}(2+\sqrt{\pi})^{\f{1}{2}}L2^{\f{\aaa+2}{2}}B_k^{\f{\aaa}{2}}\f{1}{\la^{\f{3}{2}}d^{\f{\aaa}{2}}}+\sqrt{2}\la^{-\f{1}{2}}\Big]\\
\n\cdot \Big(\mbb E_{\mathbf{z}^{|D|}}[\hc]+\mbb E_{\mathbf{z}^{|D|}}[\hc^{1/2}]\Big)\cdot2^{2p}c_p\kk\sigma^{-2p}[\kk^{2p+1}(\sqrt{C_V}M)^{2p+1}\la^{-(p+\f{3}{2})}+M^{2p+1}\la^{-1}],
\eea
in which the last estimate follows from $\|(\la I+L_{K,D})^{-1}\he\|_K\leq\la^{-1}\|\he\|_K$, Lemma \ref{ees} and the basic fact that $\sqrt{2}(2+\sqrt{\pi})^{\f{1}{2}}L2^{\f{\aaa+2}{2}}B_k^{\f{\aaa}{2}}\f{1}{\la^{\f{3}{2}}d^{\f{\aaa}{2}}}\hc+\sqrt{2}\la^{-\f{1}{2}}\hc^{1/2}\leq\big[\sqrt{2}(2+\sqrt{\pi})^{\f{1}{2}}L2^{\f{\aaa+2}{2}}B_k^{\f{\aaa}{2}}\f{1}{\la^{\f{3}{2}}d^{\f{\aaa}{2}}}+\sqrt{2}\la^{-\f{1}{2}}\big](\hc+\hc^{1/2})$ . Using Lemma \ref{bcd} to $\hc$ and Proposition \ref{fdkb}, we have
\bea
\n\mbb E[\|I_{1,2}\|_K]\\
\n\leq\kk L(2+\pi)^{\f{1}{2}}\f{2^{\f{\aaa+2}{2}}B_k^{\f{\aaa}{2}}}{d^{\f{\aaa}{2}}}\Big[\sqrt{2}(2+\sqrt{\pi})^{\f{1}{2}}L2^{\f{\aaa+2}{2}}B_k^{\f{\aaa}{2}}\f{1}{\la^{\f{3}{2}}d^{\f{\aaa}{2}}}+\sqrt{2}\la^{-\f{1}{2}}\Big]\cdot\Big[(2\Gamma(5)+\log^42)^{1/2}\hha^2\\
\n+(2\Gamma(3)+\log^22)\hha\Big]\Bigg\{2C_{p,\kk,C_V,M}\Big[\hha^2\f{\ha'}{\sqrt{\la}}+\hha^2\f{\ha}{\sqrt{\la}}\|f_{\la}\|_K\\
\n+\hha\sigma^{-2p}(\la^{-(p+\f{3}{2})}+\la^{-1})\Big]+2\|f_{\la}\|_K\Bigg\}\\
\n+\kk L(2+\pi)^{\f{1}{2}}\f{2^{\f{\aaa+2}{2}}B_k^{\f{\aaa}{2}}}{d^{\f{\aaa}{2}}}\Big[\sqrt{2}(2+\sqrt{\pi})^{\f{1}{2}}L2^{\f{\aaa+2}{2}}B_k^{\f{\aaa}{2}}\f{1}{\la^{\f{3}{2}}d^{\f{\aaa}{2}}}+\sqrt{2}\la^{-\f{1}{2}}\Big]\cdot2^{2p}c_p\kk\sigma^{-2p}\\
&&\cdot[\kk^{2p+1}(\sqrt{C_V}M)^{2p+1}\la^{-(p+\f{3}{2})}+M^{2p+1}\la^{-1}][(2\Gamma(5)+\log^42)^{\f{1}{2}}\hha^2+(2\Gamma(3)+\log^22)\hha].\label{i12}
\eea
%According to the coefficients of right hand side of .....
Now we consider norm estimates on $I_2$ in \eqref{i2}.
\bea
\n I_2=(\la I+L_{K,\hat{D}})^{-1}\hhe-(\la I+L_{K,D})^{-1}\he\\
\n =[(\la I+L_{K,\hat{D}})^{-1}-(\la I+L_{K,D})^{-1}]\hhe+(\la I+L_{K,D})^{-1}(\hhe-\he).
\eea
Then it follows that
\be
\|I_2\|_{L_{\rho_{X_{\mu}}}^2}\leq\|I_{2,1}\|_K+\|I_{2,2}\|_K, \ \text{and} \  \mbb E[\|I_2\|_{L_{\rho_{X_{\mu}}}^2}]\leq\mbb E[\|I_{2,1}\|_K]+\mbb E[\|I_{2,2}\|_K],
\ee
in which
\bea
&&I_{2,1}=L_K^{1/2}[(\la I+L_{K,\hat{D}})^{-1}-(\la I+L_{K,D})^{-1}]\hhe,\\
&&I_{2,2}=L_K^{1/2}(\la I+L_{K,D})^{-1}(\hhe-\he).
\eea
Now we estimate $\mbb E[\|I_{2,1}\|_K]$ and $\mbb E[\|I_{2,2}\|_K]$. For $\|I_{2,1}\|_K$, we have
\bea
\n\|I_{2,1}\|_K=\|L_K^{1/2}(\la I+L_{K,\hat{D}})^{-1}(L_{K,D}-L_{K,\hat{D}})(\la I+L_{K,D})^{-1}\hhe\|_K\\
&&\leq \|L_K^{1/2}[(\la I+L_{K,\hat{D}})^{-1}\|\|L_{K,D}-L_{K,\hat{D}}\|\|(\la I+L_{K,D})^{-1}\hhe\|_K. \label{i2111}
\eea
By Lemma \ref{ees}, $\|(\la I+L_{K,D})^{-1}\hhe\|_K$ shares  same upper bound of $2^{2p}c_p\kk\sigma^{-2p}[\kk^{2p+1}(\sqrt{C_V}M)^{2p+1}\la^{-(p+\f{3}{2})}+M^{2p+1}\la^{-1}]$ with $\|(\la I+L_{K,D})^{-1}\he\|_K$. Then, it is easy to see the right hand side of \eqref{i2111} has already been estimated in the procedure on estimating $\mbb [\|I_{1,2}\|_K]$, with bound
\bea
\n\mbb E[\|I_{2,1}\|_K]\leq\kk L(2+\pi)^{\f{1}{2}}\f{2^{\f{\aaa+2}{2}}B_k^{\f{\aaa}{2}}}{d^{\f{\aaa}{2}}}\Big[\sqrt{2}(2+\sqrt{\pi})^{\f{1}{2}}L2^{\f{\aaa+2}{2}}B_k^{\f{\aaa}{2}}\f{1}{\la^{\f{3}{2}}d^{\f{\aaa}{2}}}+\sqrt{2}\la^{-\f{1}{2}}\Big]\cdot2^{2p}c_p\kk\sigma^{-2p}\\
\n\cdot[\kk^{2p+1}(\sqrt{C_V}M)^{2p+1}\la^{-(p+\f{3}{2})}+M^{2p+1}\la^{-1}][(2\Gamma(5)+\log^42)^{\f{1}{2}}\hha^2+(2\Gamma(3)+\log^22)\hha].\\
&&  \label{i21}
\eea
For $\mbb E[\|I_{2,2}\|_K]$, since
\bea
\n \|I_{2,2}\|_K\leq\|L_K^{1/2}(\la I+L_{K,D})^{-1}\|\Big(\|\hhe\|_K+\|\he\|_K\Big)\\
\n\leq\|(\la I+L_K)^{1/2}(\la I+L_{K,D})^{-1/2}\|\|(\la I+L_{K,D})^{-1/2}\|\Big(\|\hhe\|_K+\|\he\|_K\Big)\\
\n\leq\hc^{1/2}\cdot\la^{-1/2}\cdot2\cdot2^{2p}c_p\kk\sigma^{-2p}\Big[\kk^{2p+1}(\sqrt{C_V}M)^{2p+1}\la^{-(p+\f{1}{2})}+M^{2p+1}\Big],
\eea
in which the second inequality follows from Lemma \ref{ees} and the fact that $\|T_1^sT_2^s\|\leq\|T_1T_2\|^s$, $s\in (0,1]$ for two positive self-adjoint operators $T_1$, $T_2$, after taking expectation, by using Lemma \ref{bcd} to $\hc^{1/2}$ with $s=1/2$, we have
\bea
\mbb E[\|I_{2,2}\|_K]\leq(2\Gamma(2)+\log2)\hha\cdot2\cdot2^{2p}c_p\kk\sigma^{-2p}\Big[\kk^{2p+1}(\sqrt{C_V}M)^{2p+1}\la^{-(p+1)}+M^{2p+1}\la^{-1/2}\Big].\label{i22}
\eea
Finally, combine the estimates \eqref{i11}, \eqref{i12}, \eqref{i21}, \eqref{i22} for $\mbb E[\|I_{1,1}\|_K]$, $\mbb E[\|I_{1,2}\|_K]$, $\mbb E[\|I_{2,1}\|_K]$, $\mbb E[\|I_{2,2}\|_K]$, then set and take out the following constants associated with them,
\bea
\n c_1=\Big[\sqrt{2}(2+\sqrt{\pi})^{\f{1}{2}}L2^{\f{2+\aaa}{2}}B_k^{\f{\aaa}{2}}(2\Gamma(3)+\log^22)+\sqrt{2}(2\Gamma(2)+\log2)\Big](2+\sqrt{\pi})^{\f{1}{2}}LM2^{\f{\aaa}{2}}B_k^{\f{\aaa}{2}},\\
\n c_2=2\kk L(2+\sqrt{\pi})^{\f{1}{2}}2^{\f{\aaa+2}{2}}B_k^{\f{\aaa}{2}}\Big(\sqrt{2}(2+\sqrt{\pi})^{\f{1}{2}}L2^{\f{\aaa+2}{2}}B_k^{\f{\aaa}{2}}+\sqrt{2}\Big)\Big[(2\Gamma(5)+\log^42)^{\f{1}{2}}\\
\n\ \ \ \ \ \ \ \ \ \  +(2\Gamma(3)+\log^22)\Big]\Big(2C_{p,\kk,C_V,M}+2\Big),\\
\n c_3=2\kk L(2+\sqrt{\pi})^{\f{1}{2}}2^{\f{\aaa+2}{2}}B_k^{\f{\aaa}{2}}\Big(\sqrt{2}(2+\sqrt{\pi})^{\f{1}{2}}L2^{\f{\aaa+2}{2}}B_k^{\f{\aaa}{2}}+\sqrt{2}\Big)\Big[(2\Gamma(5)+\log^42)^{\f{1}{2}}\\
\n\ \ \ \ \ \ \ \ \ \  +(2\Gamma(3)+\log^22)\Big]2^{2p}c_p\kk\Big[\kk^{2p+1}(\sqrt{C_V}M)^{2p+1}+M^{2p+1}\Big],\\
\n c_4=(2\Gamma(2)+\log2)2\cdot2^{2p}c_p\kk\Big[\kk^{2p+1}(\sqrt{C_V}M)^{2p+1}+M^{2p+1}\Big].
\eea
$C_{p,\kk,C_V,M}$ is defined as in \eqref{dac}. Finally, taking $\bar{C}=\max\{c_1, c_2, c_3, c_4\}$ yields the desired bound in Proposition \ref{hfjf}.
\end{proof}

As a result of combing Proposition \ref{hfjf}, Proposition \ref{ones} and \eqref{sz}, the next theorem is a general error estimate without decaying restriction on effective dimension $\huaN(\la)$. The result is crucial to obtain learning rates.

\begin{thm}\label{thm1}
Suppose that the regularity condition \eqref{regularity} holds with $r>0$ and $|y|\leq M$ almost surely.  The mapping $K_{(\cdot)}:X_{\mu}\rightarrow\mathcal{L}(Y,\huaH_K)$ is $(\aaa,L)$-H\"older continuous with $\aaa\in(0,1]$ and $L>0$. Then we have
\bea
\nono&&\mbb E\bigg[\big\|f_{\hat{D},\la}^{\sigma}-f_{\rho}\big\|_{L_{\rho_{X_{\mu}}}^2}\bigg]\\
\nono&&\leq \bar{C}\Bigg\{\big(\f{1}{\la d^{\f{\aaa}{2}}}\hha^2+\hha\big)\f{1}{\la^{\f{1}{2}}d^{\f{\aaa}{2}}}+\f{1}{\la^{\f{1}{2}}d^{\f{\aaa}{2}}}(\f{1}{\la d^{\f{\aaa}{2}}}+1)(\hha^2+\hha)\bigg[\hha^2\cdot\f{\ha'}{\sqrt{\la}}\\
\n \ +\hha^2(\hha-1)\|f_{\la}\|_K+\hha\sigma^{-2p}(\la^{-(p+\f{3}{2})}+\la^{-1})+\|f_{\la}\|_K\bigg]\\
\n \ +\f{1}{\la^{\f{1}{2}}d^{\f{\aaa}{2}}}\f{1}{\la d^{\f{\aaa}{2}}}\sigma^{-2p}(\la^{-(p+\f{3}{2})}+\la^{-1})(\hha^2+\hha)+\hha\sigma^{-2p}(\la^{-(p+1)}+\la^{-\f{1}{2}})\Bigg\}\\
&&\ +C_{p,\kk,C_V,M}\Big[\hha^2\ha'+\hha^2\ha\|f_{\la}\|_K+\hha\cdot\f{1}{\sqrt{\la}}\sigma^{-2p}(\la^{-(p+\f{1}{2})}+1)\Big]+\|g_{\rho}\|_{L_{\rho_{X_{\mu}}}^2}\la^{\min\{r,1\}}.\label{main}
\eea
$\bar{C}$ is defined as in Proposition \ref{hfjf}, $C_{p,\kk,C_V,M}$ is defined as in \eqref{dac}.  $\ha$, $\ha'$, $\hha$ are defined as in \eqref{r1}, \eqref{r2}, \eqref{r3}.
\end{thm}
\begin{proof}
Combine Proposition \ref{hfjf}, Proposition \ref{ones} and \eqref{sz}, the proof is completed.
\end{proof}
\section{Proofs of main results}
We estimate the learning rates of RDR in this section. In the following, for convenience of analysis on learning rates, we use the convention that $A_{|D|}\les B_{|D|}$ ($A_{|D|}=\o(B_{|D|})$) denotes that there exist some constant $C>0$ independent of the cardinality $|D|$, $d$ and $\sigma$ such that $A_{|D|}\les C\cdot B_{|D|}$ for any $|D|$ for some functions $A_{|D|}$, $B_{|D|}$ which may depend on $|D|$. Also, we use $A_{|D|}\les 1$ to denote that there is a constant $C>0$ independent of $|D|$, $d$ and $\sigma$ such that $A_{|D|}\leq C$.
%Also, we write $A_{|D|}\thickapprox B_{|D|}$ to denote that there are constants $C_1>0$ and $C_2>0$ independent of $|D|$ such that $C_1A_{|D|}\leq B_{|D|}\leq C_2B_{|D|}$.
We need to estimate the right hand side of \eqref{main} in Theorem \ref{thm1} in different regularity range of $r$ when the regularization parameter $\la$ and second stage sample size $d$ take different orders of $|D|$, the cardinality of data set $D$. According to Smale and Zhou \cite{s4}
\bea
\|f_{\la}\|_K&\leq&\left\{
  \begin{array}{ll}
    \|g_{\rho}\|_{L_{\rho_{X_{\mu}}}^2}\la^{r-\f{1}{2}}, & \hbox{$r\in(0,1/2)$;} \\
   \kk^{2r-1}\|g_{\rho}\|_{L_{\rho_{X_{\mu}}}^2}, & \hbox{$r\in[1/2,\infty)$;} \label{smz}
  \end{array}
\right.
\eea
this estimate will be used in the following in different regularity range of $r$. For convenience, we denote the main terms of right hand side of \eqref{main} of Theorem \ref{thm1} by
\bea
\n \tra=\Big(\f{1}{\la d^{\f{\aaa}{2}}}\hha^2+\hha\Big)\f{1}{\la^{\f{1}{2}}d^{\f{\aaa}{2}}},\\
\n\trb=\f{1}{\la^{\f{1}{2}}d^{\f{\aaa}{2}}}(\f{1}{\la d^{\f{\aaa}{2}}}+1)(\hha^2+\hha)\Big[\hha^2\f{\ha'}{\sqrt{\la}}+\hha(\hha-1)\|f_{\la}\|_K\\
\n\ \ \ \ \ \ \ \ \ \ \ \ \  +\hha\sigma^{-2p}(\la^{-(p+\f{3}{2})}+\la^{-1})+\|f_{\la}\|_K\Big],\\
\n\trc=\f{1}{\la^{\f{1}{2}}d^{\f{\aaa}{2}}}(\f{1}{\la d^{\f{\aaa}{2}}}+1)\sigma^{-2p}(\la^{-(p+\f{3}{2})}+\la^{-1})(\hha^2+\hha),\\
\n\trd=\hha\sigma^{-2p}(\la^{-(p+1)}+\la^{-1/2}),\\
\n\tre=\hha^2\ha'+\hha^2\ha\|f_{\la}\|_K+\hha\cdot\f{1}{\sqrt{\la}}\sigma^{-2p}(\la^{-(p+\f{1}{2})}+1),\\
\n\trf=\|g_{\rho}\|_{L_{\rho_{X_{\mu}}}^2}\la^{\min\{r,1\}}.
\eea
In the following, the goal is to estimate the main terms $\tra\sim\trf$ for different regularity range of $r$ when $\la$ and $d$ take different orders  of $|D|$.
\subsection{Learning rates for $r\in(0,1/2)$}
When $r\in(0,1/2)$, \eqref{smz} implies $\|f_{\la}\|_K\leq\|g_{\rho}\|_{L_{\rho_{X_{\mu}}}^2}\la^{r-\f{1}{2}}$, $r\in(0,1/2)$. After taking  $\la=|D|^{-\f{1}{1+\beta}}$, $d=|D|^{\f{2}{\aaa(1+\beta)}}$, the following basic estimates hold for $r\in(0,1/2)$:
\be
\nono\f{1}{\la^{\f{1}{2}}d^{\f{\aaa}{2}}}=|D|^{\f{1}{2(1+\beta)}}|D|^{-\f{1}{1+\beta}}=|D|^{-\f{1}{2(1+\beta)}},
\ee
\be
\nono \f{1}{\la d^{\f{\aaa}{2}}}=|D|^{\f{1}{1+\beta}}|D|^{-\f{1}{1+\beta}}=1.
\ee
According to the condition $\huaN(\la)\leq\c_0\la^{-\beta}$, $\beta\in(0,1]$, we have
\be
\nono\ha=\f{2\kk}{\sqrt{|D|}}\Big(\f{\kk}{\sqrt{|D|\la}}+\sqrt{\huaN(\la)}\Big)\leq2\kk(\kk+\sqrt{\c_0})|D|^{-\f{r}{1+\beta}}\les|D|^{-\f{r}{1+\beta}}.
\ee
and
\be
\nono\f{\ha}{\sqrt{\la}}\leq2\kk(\kk+\sqrt{\c_0}), \ \hha=\f{\ha}{\sqrt{\la}}+1\leq 2\kk(\kk+\sqrt{\c_0})+1.
\ee
Same procedure with above inequalities implies
\be
\nono\ha'\les|D|^{-\f{r}{1+\beta}} \  \text{and} \ \f{\ha'}{\sqrt{\la}}\leq1+\sqrt{\c_0}.
\ee
Then,  since $\hha\les 1$ and $\hha-1\les 1$, it follows that
\be
\nono \tra=\Big(\f{1}{\la d^{\f{\aaa}{2}}}\hha^2+\hha\Big)\f{1}{\la^{\f{1}{2}}d^{\f{\aaa}{2}}}\les|D|^{-\f{1}{2(1+\beta)}}\les|D|^{-\f{r}{1+\beta}}  \  (\text{since} \ r\in(0,1/2)).
\ee
Now turn to $\trb$. We spit $\trb$ into four parts and estimate them each other.
\be
\nono(i)\ \ \ \f{1}{\la^{\f{1}{2}}d^{\f{\aaa}{2}}}(\f{1}{\la d^{\f{\aaa}{2}}}+1)(\hha^2+\hha)\hha^2\f{\ha'}{\sqrt{\la}}\les|D|^{-\f{1}{2(1+\beta)}}\les|D|^{-\f{r}{1+\beta}}.
\ee
When $\la=|D|^{-\f{1}{1+\beta}}$, $d=|D|^{\f{2}{\aaa(1+\beta)}}$,
\bea
\n (ii) \ \ \ \ \f{1}{\la^{\f{1}{2}}d^{\f{\aaa}{2}}}(\f{1}{\la d^{\f{\aaa}{2}}}+1)(\hha^2+\hha)\hha^2(\hha-1)\|f_{\la}\|_K\les \f{1}{\la^{\f{1}{2}}d^{\f{\aaa}{2}}}\cdot\la^{r-\f{1}{2}}\\
\n\quad\quad\quad\quad\quad\quad\quad\quad\quad\quad\quad\quad\quad\quad\quad\quad\quad\quad\quad\quad\les|D|^{-\f{1}{2(1+\beta)}}|D|^{-\f{r-\f{1}{2}}{1+\beta}}\les|D|^{-\f{r}{1+\beta}}.
\eea
Same way with $(ii)$ implies
\be
\nono (iii) \ \ \ \ \ \f{1}{\la^{\f{1}{2}}d^{\f{\aaa}{2}}}(\f{1}{\la d^{\f{\aaa}{2}}}+1)(\hha^2+\hha)\|f_{\la}\|_K\les|D|^{-\f{r}{1+\beta}}.
\ee
For the forth term,
\bea
\n (iv)\ \ \ \ \ \ \f{1}{\la^{\f{1}{2}}d^{\f{\aaa}{2}}}(\f{1}{\la d^{\f{\aaa}{2}}}+1)(\hha^2+\hha)\hha\sigma^{-2p}(\la^{-(p+\f{3}{2})}+\la^{-1})\\
\n\quad\quad\quad\quad\les |D|^{-\f{1}{2(1+\beta)}}\cdot 1\cdot \sigma^{-2p}\Big(|D|^{\f{p+\f{3}{2}}{1+\beta}}+|D|^{\f{1}{1+\beta}}\Big)\\
\n\quad\quad\quad\quad\les \sigma^{-2p}\Big(|D|^{\f{p+1}{1+\beta}}+|D|^{\f{1}{2(1+\beta)}}\Big)\les\max\Big\{\f{|D|^{\f{p+1}{1+\beta}}}{\sigma^{2p}},\f{|D|^{\f{1}{2(1+\beta)}}}{\sigma^{2p}}\Big\}\les\f{|D|^{\f{p+1}{1+\beta}}}{\sigma^{2p}}.
\eea
Combining $(i)\sim(iv)$, we obtain that
\be
\nono\trb\les\max\Big\{|D|^{-\f{r}{1+\beta}},\f{|D|^{\f{p+1}{1+\beta}}}{\sigma^{2p}}\Big\}.
\ee
Then the same way with above $(iv)$ implies
\be
\nono\trc=\f{1}{\la^{\f{1}{2}}d^{\f{\aaa}{2}}}(\f{1}{\la d^{\f{\aaa}{2}}}+1)\sigma^{-2p}(\la^{-(p+\f{3}{2})}+\la^{-1})(\hha^2+\hha)\les\f{|D|^{\f{p+1}{1+\beta}}}{\sigma^{2p}}.
\ee
For $\trd$, it follows that
\be
\nono\trd=\hha\sigma^{-2p}(\la^{-(p+1)}+\la^{-1/2})\leq\sigma^{-2p}\Big(|D|^{\f{p+1}{1+\beta}}+|D|^{\f{1}{2(1+\beta)}}\Big)\les\f{|D|^{\f{p+1}{1+\beta}}}{\sigma^{2p}} \ (\text{since} \ p+1>\f{1}{2}).
\ee
For $\tre$, since
\be
\nono\hha^2\ha'\les|D|^{-\f{1}{2(1+\beta)}}\les|D|^{-\f{r}{1+\beta}},
\ee
\be
\nono\hha^2\ha\|f_{\la}\|_K\leq\hha^2\ha\|g_{\rho}\|_{L_{\rho_{X_{\mu}}}^2}\la^{r-\f{1}{2}}\les|D|^{-\f{1}{2(1+\beta)}}|D|^{\f{r-\f{1}{2}}{1+\beta}}\les|D|^{-\f{r}{1+\beta}},
\ee
\be
\nono\hha\cdot\f{1}{\sqrt{\la}}\sigma^{-2p}(\la^{-(p+\f{1}{2})}+1)\les\sigma^{-2p}\Big(|D|^{\f{1+p}{1+\beta}}+|D|^{\f{1}{2(1+\beta)}}\Big)\les\f{|D|^{\f{1+p}{1+\beta}}}{\sigma^{2p}},
\ee
it follows that
\be
\nono \tre\les\max\Big\{|D|^{-\f{r}{1+\beta}},\f{|D|^{\f{p+1}{1+\beta}}}{\sigma^{2p}}\Big\}.
\ee
Then, for $\trf$,
\be
\nono \trf=\|g_{\rho}\|_{L_{\rho_{X_{\mu}}}^2}\la^{\min\{r,1\}}=\|g_{\rho}\|_{L_{\rho_{X_{\mu}}}^2}\les|D|^{-\f{r}{1+\beta}}.
\ee
Finally, combining above estimates for $\tra\sim\trf$ yields
\be
\nono \mbb E\bigg[\big\|f_{\hat{D},\la}^{\sigma}-f_{\rho}\big\|_{L_{\rho_{X_{\mu}}}^2}\bigg]=\o\Big(\max\Big\{|D|^{-\f{r}{1+\beta}},\f{|D|^{\f{p+1}{1+\beta}}}{\sigma^{2p}}\Big\}\Big), \ r\in(0,1/2).
\ee
\subsection{Learning rates for $r\in[1/2,1]$}
When $r\in[1/2,1]$, \eqref{smz} implies $\|f_{\la}\|_K\leq\kk^{2r-1}\|g_{\rho}\|_{L_{\rho_{X_{\mu}}}^2}$, $r\in [1/2,1]$. Before estimating the terms in \eqref{main}, when $\la=|D|^{-\f{1}{2r+\beta}}$, $d=|D|^{\f{1+2r}{\aaa(2r+\beta)}}$, we derive the following basic estimates at first. After substituting  $\la$ and $d$, we have
\be
\nono\f{1}{\la^{\f{1}{2}}d^{\f{\aaa}{2}}}=|D|^{-\f{r}{2r+\beta}}, \ \f{1}{\la d^{\f{\aaa}{2}}}=|D|^{\f{1}{2r+\beta}}|D|^{-\f{1+2r}{2(2r+\beta)}}=|D|^{\f{1-2r}{2(2r+\beta)}}\leq1.
\ee
Use the condition $\huaN(\la)\leq\c_0\la^{-\beta}$, $\beta\in(0,1]$, it follows that
\be
\nono\ha=\f{2\kk}{\sqrt{|D|}}\Big(\f{\kk}{\sqrt{|D|\la}}+\sqrt{\huaN(\la)}\Big)\leq2\kk(\kk+\sqrt{\c_0})|D|^{-\f{r}{2r+\beta}}\les|D|^{-\f{r}{2r+\beta}}.
\ee
and
\be
\nono\hha=\f{\ha}{\sqrt{\la}}+1=\f{2\kk}{\sqrt{|D|\la}}\Big(\f{\kk}{\sqrt{|D|\la}}+\sqrt{\huaN(\la)}\Big)+1\leq2\kk(\kk+\sqrt{\c_0})+1.
\ee
In a same way with estimate of above $\ha$, $\hha$, we have
\be
\nono\ha'\les|D|^{-\f{r}{2r+\beta}} \ \ \text{and} \ \ \f{\ha'}{\sqrt{\la}}\leq\f{1}{|D|\la}+\f{\sqrt{\huaN(\la)}}{\sqrt{|D|\la}}\leq1+\sqrt{\c_0}.
\ee
With above basic estimates, we can estimate main terms in \eqref{main} in Theorem \ref{thm1} to derive the learning rates of RDR when $r\in[1/2,1]$. Since $\hha\les 1$, $\hha-1\les 1$ and $\f{\ha'}{\sqrt{\la}}\les 1$, it follows that
\be
\nono\tra=\Big(\f{1}{\la d^{\f{\aaa}{2}}}\hha^2+\hha\Big)\f{1}{\la^{\f{1}{2}}d^{\f{\aaa}{2}}}\les|D|^{-\f{r}{2r+\beta}}.
\ee
Also,
\bea
\n\trb=\f{1}{\la^{\f{1}{2}}d^{\f{\aaa}{2}}}\f{1}{\la d^{\f{\aaa}{2}}}(\hha^2+\hha)\Big[\hha^2\f{\ha'}{\sqrt{\la}}+\hha(\hha-1)\|f_{\la}\|_K\\
\n+\hha\sigma^{-2p}(\la^{-(p+\f{3}{2})}+\la^{-1})+\|f_{\la}\|_K\Big]\les|D|^{-\f{r}{2r+\beta}}+\f{|D|^{\f{(\f{3}{2}+p)-r}{2r+\beta}}}{\sigma^{2p}}+\f{|D|^{\f{1-r}{2r+\beta}}}{\sigma^{2p}}\\
\nono&& \quad\quad\quad\quad\quad\quad\quad\quad\quad\quad\quad\quad\quad\quad\quad\quad\quad\quad\les\max\Big\{|D|^{-\f{r}{2r+\beta}},\f{|D|^{(\f{3}{2}-r)+p}}{\sigma^{2p}}\Big\},
\eea
in which the last inequality follows from the fact $\f{3}{2}+p>1$. For $\trc$, we have
\be
\nono\trc=\f{1}{\la^{\f{1}{2}}d^{\f{\aaa}{2}}}(\f{1}{\la d^{\f{\aaa}{2}}}+1)\sigma^{-2p}(\la^{-(p+\f{3}{2})}+\la^{-1})(\hha^2+\hha)\les\max\Big\{\f{|D|^{\f{\f{3}{2}+p-r}{2r+\beta}}}{\sigma^{2p}},\f{|D|^{\f{1-r}{2r+\beta}}}{\sigma^{2p}}\Big\}\les\f{|D|^{\f{\f{3}{2}+p-r}{2r+\beta}}}{\sigma^{2p}}.
\ee
Also, for $\trd$, we have
\be
\nono\trd=\hha\sigma^{-2p}(\la^{-(p+1)}+\la^{-1/2})\les\sigma^{-2p}\Big(|D|^{\f{p+1}{2r+\beta}}+|D|^{\f{1}{2}(\f{1}{2r+\beta})}\Big)\les\f{|D|^{\f{p+1}{2r+\beta}}}{\sigma^{2p}} \ (\text{since} \ p+1>\f{1}{2}).
\ee
On the other hand, since
\be
\nono\hha^2\ha'\les|D|^{-\f{r}{2r+\beta}},
\ee
\be
\nono\hha\ha\|f_{\la}\|_K\leq\hha\ha\kk^{2r-1}\|g_{\rho}\|_{L_{\rho_{X_{\mu}}}^2}\les|D|^{-\f{r}{2r+\beta}},
\ee
\be
\nono\hha\f{1}{\sqrt{\la}}\sigma^{-2p}(\la^{-(p+\f{1}{2})}+1)\les\f{|D|^{\f{p+1}{2r+\beta}}}{\sigma^{2p}}+\f{|D|^{\f{1}{2}(\f{1}{2r+\beta})}}{\sigma^{2p}}\les\f{|D|^{\f{p+1}{2r+\beta}}}{\sigma^{2p}},
\ee
Therefore, we have
\be
\nono\tre\les|D|^{-\f{r}{2r+\beta}}+\f{|D|^{\f{p+1}{2r+\beta}}}{\sigma^{2p}}\les\max\{|D|^{-\f{r}{2r+\beta}},\f{|D|^{\f{p+1}{2r+\beta}}}{\sigma^{2p}}\}.
\ee
Also, $\trf$ is estimated as follow
\be
\nono\trf=\nono\|g_{\rho}\|_{L_{\rho_{X_{\mu}}}^2}\la^{\min\{r,1\}}=\|g_{\rho}\|_{L_{\rho_{X_{\mu}}}^2}\la^r\les|D|^{-\f{r}{2r+\beta}}.
\ee
Now combining above estimates for the main terms $\tra\sim\trf$, noting the fact that when $r\in[1/2,1]$, $\f{3}{2}-r\leq1$ and
\be
\nono |D|^{\f{\f{3}{2}+p-r}{2r+\beta}}/\sigma^{2p}\leq |D|^{\f{1+p}{2r+\beta}}/\sigma^{2p},
\ee
we finally have
\be
\nono \mbb E\bigg[\big\|f_{\hat{D},\la}^{\sigma}-f_{\rho}\big\|_{L_{\rho_{X_{\mu}}}^2}\bigg]=\o\Big(\max\Big\{|D|^{-\f{r}{2r+\beta}},\f{|D|^{\f{1+p}{2r+\beta}}}{\sigma^{2p}}\Big\}\Big), \ r\in[1/2,1].
\ee

\subsection{Learning rates for $r\in(1,\infty)$}
When $r\in(0,\infty)$, \eqref{smz} implies $\|f_{\la}\|_K\leq\kk^{2r-1}\|g_{\rho}\|_{L_{\rho_{X_{\mu}}}^2}$, $r>1$. After taking $\la=|D|^{-\f{1}{2+\beta}}$, $d=|D|^{\f{1}{\aaa}(\f{3}{2+\beta})}$, we start with following basic estimates:
\be
\nono\f{1}{\la^{\f{1}{2}}d^{\f{\aaa}{2}}}\leq|D|^{\f{1}{2(2+\beta)}}|D|^{-\f{3}{2(2+\beta)}}=|D|^{-\f{1}{2+\beta}},
\ee
\be
\nono\f{1}{\la d^{\f{\aaa}{2}}}\leq|D|^{\f{1}{2+\beta}}|D|^{-\f{3}{2(2+\beta)}}=|D|^{\f{2}{2(2+\beta)}-\f{3}{2(2+\beta)}}=|D|^{-\f{1}{2(2+\beta)}}<1.
\ee
Since $\huaN(\la)\leq \c_0\la^{-\beta}$, $\beta\in(0,1]$, it follows that
\be
\nono\ha'=\f{1}{|D|\sqrt{\la}}+\f{\sqrt{\huaN(\la)}}{\sqrt{|D|}}\leq(1+\sqrt{\c_0})\f{\la^{-\f{\beta}{2}}}{\sqrt{|D|}}\Big(\f{1}{\sqrt{|D|}}\la^{\f{\beta-1}{2}}+1\Big).
\ee
When $\la=|D|^{-\f{1}{2+\beta}}$, $\f{1}{\sqrt{|D|}}\la^{\f{\beta-1}{2}}=|D|^{-\f{1}{2}}|D|^{\f{1-\beta}{2}\f{1}{2+\beta}}=|D|^{-\f{2\beta+1}{2(2+\beta)}}<1$ and $\f{\la^{-\f{\beta}{2}}}{\sqrt{|D|}}=|D|^{\f{\beta}{2(2+\beta)}-\f{1}{2}}=|D|^{\f{\beta-2-\beta}{2(2+\beta)}}=|D|^{-\f{1}{2+\beta}}$, hence we have
\be
\nono\ha'\les|D|^{-\f{1}{2+\beta}}, \ \ \text{same way implies} \ \ \ha\les|D|^{-\f{1}{2+\beta}}.
\ee
Also,
\be
\nono\f{\ha'}{\sqrt{\la}}\leq(1+\sqrt{\c_0})\f{|D|^{-\f{1}{2+\beta}}}{|D|^{-\f{1}{2(2+\beta)}}}\leq(1+\sqrt{\c_0})|D|^{-\f{1}{2(2+\beta)}}\leq1+\sqrt{\c_0}.
\ee
Same way implies
\be
\nono \f{\ha}{\sqrt{\la}}\leq1+\sqrt{\c_0}.
\ee
Now based on above estimates, note that $\hha\les 1$, $\hha-1\les 1$ and $\f{\ha'}{\sqrt{\la}}\les 1$, we can estimate $\tra\sim\trf$ as follows,
\be
\nono\tra=\Big(\f{1}{\la d^{\f{\aaa}{2}}}\hha^2+\hha^2\Big)\f{1}{\la^{\f{1}{2}}d^{\f{\aaa}{2}}}\les|D|^{-\f{1}{2+\beta}}.
\ee
For $\trb$, since
\be
\nono(i)\ \ \ \ \ \ \ \f{1}{\la^{\f{1}{2}}d^{\f{\aaa}{2}}}(\f{1}{\la d^{\f{\aaa}{2}}}+1)(\hha^2+\hha)\hha^2\f{\ha'}{\sqrt{\la}}\les|D|^{-\f{1}{2+\beta}},
\ee
\be
\nono(ii)\ \ \ \ \ \ \ \f{1}{\la^{\f{1}{2}}d^{\f{\aaa}{2}}}(\f{1}{\la d^{\f{\aaa}{2}}}+1)(\hha^2+\hha)\Big[\hha(\hha-1)\|f_{\la}\|_K+\|f_{\la}\|_K\Big]\les|D|^{-\f{1}{2+\beta}}.
\ee
\bea
\n(iii)\ \   \f{1}{\la^{\f{1}{2}}d^{\f{\aaa}{2}}}(\f{1}{\la d^{\f{\aaa}{2}}}+1)(\hha^2+\hha)\hha\sigma^{-2p}(\la^{-(p+\f{3}{2})}+\la^{-1})\les\sigma^{-2p}\Big(|D|^{\f{p+\f{3}{2}-1}{2+\beta}}+1\Big)\les\f{|D|^{\f{p+\f{1}{2}}{2+\beta}}}{\sigma^{2p}},
\eea
it follows that
\be
\nono\trb\les\max\Big\{|D|^{-\f{1}{2+\beta}},\f{|D|^{\f{p+\f{1}{2}}{2+\beta}}}{\sigma^{2p}}\Big\}.
\ee
Same reason with $(iii)$ implies
\be
\nono \trc=\f{1}{\la^{\f{1}{2}}d^{\f{\aaa}{2}}}(\f{1}{\la d^{\f{\aaa}{2}}}+1)\sigma^{-2p}(\la^{-(p+\f{3}{2})}+\la^{-1})(\hha^2+\hha)\les\f{|D|^{\f{p+\f{1}{2}}{2+\beta}}}{\sigma^{2p}}.
\ee
For $\trd$, we have
\be
\nono\trd=\hha\sigma^{-2p}(\la^{-(p+1)}+\la^{-1/2})\les\sigma^{-2p}\Big(|D|^{\f{p+1}{2+\beta}}+|D|^{\f{1}{2(2+\beta)}}\Big)\les\f{|D|^{\f{p+1}{2+\beta}}}{\sigma^{2p}}.
\ee
Also, since
\be
\nono\hha^2\ha'\les|D|^{-\f{1}{2+\beta}},
\ee
\be
\nono\hha^2\ha\|f_{\la}\|_K\leq\hha^2\ha\kk^{2r-1}\|g_{\rho}\|_{L_{\rho_{X_{\mu}}}^2}\les\ha\les|D|^{-\f{1}{2+\beta}},
\ee
\be
\hha\f{1}{\sqrt{\la}}\sigma^{-2p}(\la^{-(p+\f{1}{2})}+1)\les\f{1}{\sqrt{\la}}\sigma^{-2p}(\la^{-(p+\f{1}{2})}+1)\les\sigma^{-2p}(|D|^{\f{p+1}{2+\beta}}+|D|^{\f{1}{2(2+\beta)}})\les\f{|D|^{\f{p+1}{2+\beta}}}{\sigma^{2p}},
\ee
it follows that
\be
\nono\tre\les\max\Big\{|D|^{-\f{1}{2+\beta}},\f{|D|^{\f{p+1}{2+\beta}}}{\sigma^{2p}}\Big\}.
\ee
Finally, since $r>1$, we have
\be
\nono\trf=\|g_{\rho}\|_{L_{\rho_{X_{\mu}}}^2}\la^{\min\{r,1\}}=\|g_{\rho}\|_{L_{\rho_{X_{\mu}}}^2}\la^1=\|g_{\rho}\|_{L_{\rho_{X_{\mu}}}^2}\cdot|D|^{-\f{1}{2+\beta}}\les|D|^{-\f{1}{2+\beta}}.
\ee
Combining above estimates for $\tra\sim\trf$, we obtain
\be
 \mbb E\bigg[\big\|f_{\hat{D},\la}^{\sigma}-f_{\rho}\big\|_{L_{\rho_{X_{\mu}}}^2}\bigg]=\o\Big(\max\Big\{|D|^{-\f{1}{2+\beta}},\f{|D|^{\f{p+1}{2+\beta}}}{\sigma^{2p}}\Big\}\Big).
\ee
Now it is ready to provide proof for Corollary \ref{sigm}.

\subsection{Proof of Corollary \ref{sigm}}
\emph{Proof.} The proof is obvious after using Theorem \ref{thm2} and noting the fact that, if
\bea
\sigma&\geq&\left\{
  \begin{array}{ll}
    |D|^{\f{p+1+r}{2p(1+\beta)}}, & \hbox{$r\in(0,1/2)$;} \\
   |D|^{\f{p+1+r}{2p(2r+\beta)}}, & \hbox{$r\in[1/2,1]$;}\\
   |D|^{\f{p+1+r}{2p(2+\beta)}}, & \hbox{$r\in(1,\infty)$,}
  \end{array}
\right.
\eea
then
\bea
\nono &&|D|^{-\f{r}{1+\beta}}\geq\f{|D|^{\f{p+1}{1+\beta}}}{\sigma^{2p}}, \ r\in (0,1/2);\\
\nono &&|D|^{-\f{r}{2r+\beta}}\geq\f{|D|^{\f{p+1}{2r+\beta}}}{\sigma^{2p}},\ r\in [1/2,1];\\
\nono &&|D|^{-\f{r}{2+\beta}}\geq\f{|D|^{\f{p+1}{2+\beta}}}{\sigma^{2p}}, \ r\in(1,\infty).
\eea

\subsection{Proof of Theorem \ref{quan}}
\emph{Proof.} It follows from Fang et al. \cite{fang} that the least square distribution regressor has the form
\be
\nono f_{\hat{D},\la}^{ls}=(\la I+L_{K,\hat{D}})^{-1}\hat{S}_{D}^Ty.
\ee
With the representation of $f_{\hat{D},\la}^{\sigma}$ in Lemma \ref{rel}, we have
\bea
\nono&& \|f_{\hat{D},\la}^{\sigma}-f_{\hat{D},\la}^{ls}\|_{L_{\rho_{X_{\mu}}}^2}=\|(\la I+L_{K,\hat{D}})^{-1}\hhe\|_{L_{\rho_{X_{\mu}}}^2}=\|L_K^{1/2}(\la I+L_{K,\hat{D}})^{-1}\hhe\|_K\\
\n\leq\|L_K^{1/2}(\la I+L_{K,\hat{D}})^{-1}\|\|\hhe\|_K.
\eea
Use Lemma \ref{ees} and take expectation on both sides of above inequality, it follows that
\bea
\n \mbb E\bigg[\big\|f_{\hat{D},\la}^{\sigma}-f_{\hat{D},\la}^{ls}\big\|_{L_{\rho_{X_{\mu}}}^2}\bigg]\\
\n\leq 2^{2p}c_p\kk\sigma^{-2p}\Big[\kk^{2p+1}(\sqrt{C_V}M)^{2p+1}\la^{-(p+\f{1}{2})}+M^{2p+1}\Big]\mbb E_{\mathbf{z}^{|D|}}\Bigg[\Big\{\mbb E_{\mathbf{x}^{\mathbf{d},|D|}|\mathbf{z}^{|D|}}[\|L_K^{1/2}(\la I+ L_{K,\hat{D}})^{-1}\|^2]\Big\}^{1/2} \Bigg]\\
\n\leq 2^{2p}c_p\kk\sigma^{-2p}\Big[\kk^{2p+1}(\sqrt{C_V}M)^{2p+1}\la^{-(p+\f{1}{2})}+M^{2p+1}\Big]\\
\n\ \ \cdot \Bigg[\Big(\sqrt{2}\la^{-\f{3}{2}}\kk(2+\sqrt{\pi})^{\f{1}{2}}L\f{2^{\f{\aaa+2}{2}}B_k^\f{\aaa}{2}}{d^{\f{\aaa}{2}}}\Big)\mbb E_{\mathbf{z}^{|D|}}[\hc]+\sqrt{2}\la^{-1/2}\mbb E_{\mathbf{z}^{|D|}}[\hc^{1/2}]\Bigg],
\eea
in which the first inequality follows from the basic fact that $\mbb E_{\mathbf{x}^{\mathbf{d},|D|}|\mathbf{z}^{|D|}}[\|L_K^{1/2}(\la I+ L_{K,\hat{D}})^{-1}\|]\Big\}\leq\Big\{\mbb E_{\mathbf{x}^{\mathbf{d},|D|}|\mathbf{z}^{|D|}}[\|L_K^{1/2}(\la I+ L_{K,\hat{D}})^{-1}\|^2]\Big\}^{1/2}$, the second inequality follows from Lemma \ref{secl}. After using Lemma \ref{bcd} to $\hc$ and taking out corresponding coefficients by setting
\be
\ww{C}=2^{2p}c_p\kk[\kk^{2p+1}(\sqrt{C_V}M)^{2p+1}+M^{2p+1}]\Big\{\sqrt{2}(2+\sqrt{\pi})^{\f{1}{2}}L2^{\f{\aaa+2}{2}}B_k^{\f{\aaa}{2}}(2\Gamma(3)+\log^22)+\sqrt{2}(2\Gamma(2)+\log2)\Big\},
\ee
we arrive at
\be
\nono\mbb E\bigg[\big\|f_{\hat{D},\la}^{\sigma}-f_{\hat{D},\la}^{ls}\big\|_{L_{\rho_{X_{\mu}}}^2}\bigg]\leq\ww{C}\f{(\la^{-(p+\f{1}{2})}+1)(\la^{-\f{3}{2}}d^{-\f{\aaa}{2}}\hha^2+\la^{-\f{1}{2}}\hha)}{\sigma^{2p}},
\ee
which completes the proof.

\subsection{Proof of Corollary \ref{gap}}
\emph{Proof.} For any given $\la>0$, note that
\be
\nono \lim_{|D|\rightarrow\infty}\ha=\f{2\kk}{\sqrt{|D|}}\bigg(\f{\kk}{\sqrt{|D|\la}}+\sqrt{\huaN(\la)}\bigg)=0,
\ee
then it follows that
\be
\lim_{|D|\rightarrow\infty}\hha=\lim_{|D|\rightarrow\infty}\Big(\f{\ha}{\sqrt{\la}}+1\Big)=1, \label{co1}
\ee
and
\be
\overline{\lim}_{|D|\rightarrow\infty \atop d\rightarrow\infty}d^{-\f{\aaa}{2}}\hha^2=0. \label{co2}
\ee
From Theorem \ref{quan}, we have known that
\be
\nono \mbb E\bigg[\big\|f_{\hat{D},\la}^{\sigma}-f_{\hat{D},\la}^{ls}\big\|_{L_{\rho_{X_{\mu}}}^2}\bigg]\leq\ww{C}\f{(\la^{-(p+\f{1}{2})}+1)(\la^{-\f{3}{2}}d^{-\f{\aaa}{2}}\hha^2+\la^{-\f{1}{2}}\hha)}{\sigma^{2p}}.
\ee
By taking upper limit with respect to $|D|\rightarrow\infty \atop d\rightarrow\infty$ on above inequality and using \eqref{co1} and \eqref{co2}, we obtain
\be
\nono\overline{\lim}_{|D|\rightarrow\infty \atop d\rightarrow\infty}\mbb E\bigg[\big\|f_{\hat{D},\la}^{\sigma}-f_{\hat{D},\la}^{ls}\big\|_{L_{\rho_{X_{\mu}}}^2}\bigg]\leq\ww{C}\f{(\la^{-(p+1)}+\la^{-\f{1}{2}})}{\sigma^{2p}},
\ee
which completes the proof.

\section*{Acknowledgements}
The work described in this paper is supported partially by the Research Grants Council of Hong Kong [Project No. CityU 11202819], the CityU Strategic Grant 7005511. The last author is supported partially by the Research Grants Council of Hong Kong [Project \# CityU 11307319], Hong Kong Institute for Data Science, and National Science Foundation of China [Project No. 12061160462]. This paper was written when the last author visited SAMSI/Duke during his sabbatical leave. He would like to express his gratitude to their hospitality and financial support.

\end{document}